\documentclass[10pt,twocolumn,letterpaper]{article}

\usepackage{titlesec}
\titlespacing*{\paragraph}
    {0pt}               
    {0.5\baselineskip}  
    {1em}               

\usepackage[pagenumbers]{cvpr} 

\definecolor{cvprblue}{rgb}{0.21,0.49,0.74}
\usepackage[pagebackref,breaklinks,colorlinks,allcolors=cvprblue]{hyperref}
\definecolor{forestgreen}{RGB}{75,159,101}
\def\paperID{13909} 
\def\confName{CVPR}
\def\confYear{2026}

\title{Dynamical Implicit Neural Representations}

\author{Yesom Park$^{1}$, Kelvin Kan$^{1}$, Thomas Flynn$^{2}$, Yi Huang$^{2}$, Shinjae Yoo$^{2}$, Stanley Osher$^{1}$, Xihaier Luo$^{2}$\\
$^{1}$University of California, Los Angeles \quad $^{2}$Brookhaven National Laboratory \\
{\tt\small \{yeisom, kelvin.kan, sjo\}@math.ucla.edu \{yhuang2, tflynn, sjyoo, xluo\}@bnl.gov}
}

\begin{document}
\maketitle
\begin{abstract}
Implicit Neural Representations (INRs) provide a powerful continuous framework for modeling complex visual and geometric signals, but spectral bias remains a fundamental challenge, limiting their ability to capture high-frequency details. Orthogonal to existing remedy strategies, we introduce \textbf{Dynamical Implicit Neural Representations (DINR)}, a new INR modeling framework that treats feature evolution as a continuous-time dynamical system rather than a discrete stack of layers. This dynamical formulation mitigates spectral bias by enabling richer, more adaptive frequency representations through continuous feature evolution. Theoretical analysis based on Rademacher complexity and the Neural Tangent Kernel demonstrates that DINR enhances expressivity and improves training dynamics. Moreover, regularizing the complexity of the underlying dynamics provides a principled way to balance expressivity and generalization. Extensive experiments on image representation, field reconstruction, and data compression confirm that DINR delivers more stable convergence, higher signal fidelity, and stronger generalization than conventional static INRs. \href{https://xihaier.github.io/projects/CVPR-2026-DINR/}{[Project Homepage: Data \& Code]}
\end{abstract}    
\section{Introduction}
\label{sec:intro}
Implicit Neural Representations (INRs) have emerged as a powerful paradigm for modeling complex signals across vision and geometry, representing data continuously as functions of input coordinates \cite{sitzmann2020implicit,park2019deepsdf,mescheder2019occupancy,xie2022neural,ashkenazi2024towards}. By parameterizing signals with neural networks rather than discrete grids, INRs enable high-fidelity reconstruction, continuous querying, and compact storage of diverse data modalities \cite{mildenhall2021nerf,tancik2020fourier,chen2021learning,ben2024neural,jayasundara2025sinr,gielisse2025end}. Despite their flexibility, conventional INRs suffer from \emph{spectral bias}, favoring low-frequency components and limiting accurate reconstruction of high-frequency details~\cite{rahaman2019spectral,tancik2020fourier,sitzmann2020implicit,luo2024continuous}.

We introduce \emph{Dynamical Implicit Neural Representations} (DINR), a new class of INR architectures that model latent feature evolution as a \emph{continuous-time dynamical system} rather than a standard INR’s single-pass feedforward network. Unlike standard INRs, which attempt to directly produce the target function in a single forward pass, DINR captures the \emph{dynamics of feature evolution}, allowing latent representations to progressively transform along a continuous trajectory. Intuitively, this can be understood as specifying the \emph{rate of change} of features rather than their final state, enabling exploration of richer intermediate states, coupling features across depth, and supporting more adaptive frequency representations.

We theoretically show this dynamical formulation enhances both expressivity and trainability. DINR increases the Rademacher complexity relative to conventional INRs, allowing compact networks to capture higher-frequency components and more intricate functions \cite{wei2019data,bartlett2002rademacher,mohri2018foundations}. This arises from the incremental updates along the latent trajectory, which progressively accumulate representational power. Additionally, the stepwise evolution of features produces more diverse gradient directions, increasing the effective rank of the Jacobian and associated neural tangent kernel (NTK), which accelerates learning of high-frequency modes, improves optimization, and ensures robust convergence \cite{jacot2018neural,cao2019towards}. Importantly, regulating the dynamical complexity allows DINR to maintain high expressivity while preserving strong generalization.

Experimentally, we use four diverse datasets to demonstrate that DINR consistently and significantly outperforms conventional INRs, such as FFNets \cite{tancik2020fourier} and SIREN \cite{sitzmann2020implicit}. On challenging tasks, including 2D image representation, 3D field reconstruction, and scientific data compression, our dynamical formulation yields substantial gains in both quantitative metrics and qualitative fidelity. Analyses of condition numbers, Jacobian ranks, NTK spectra, and ablations on the kinetic energy loss further confirm the incremental dynamical structure underlies DINR's superior trainability, robustness to noise, and generalization.

\paragraph{In summary, our main contributions are:}
\begin{itemize}
    \item We introduce DINR, a new paradigm for implicit representations that replaces static MLP transformations with controlled latent dynamics, to unlock greater expressivity.
    \item We theoretically show DINR enhances representational power, gradient diversity, and NTK rank, while regulating dynamical complexity preserves generalization.
    \item We empirically validate DINR on four diverse datasets, demonstrating consistent improvements over conventional INRs in high-frequency reconstruction, parameter efficiency, convergence speed, and robustness.
\end{itemize}
\section{Related Work}
\label{sec:relatedwork}

The proposed DINR is a model-agnostic framework that synthesizes principles from INRs and continuous-time modeling. To properly situate our contributions, which include a rigorous theoretical analysis, we now review the most relevant literature from these domains.

\paragraph{Implicit Neural Representations.} INRs constitute a novel class of models that encode data as continuous functions via neural networks~\cite{xie2022neural}. INRs' modality-agnostic nature has facilitated a broad range of applications across diverse fields, such as signal representation, e.g., 2D images~\cite{chen2021learning}, 3D scenes~\cite{park2019deepsdf}, videos~\cite{he2023towards}, and neural rendering~\cite{mildenhall2021nerf,muller2022instant}. However, widespread adoption of INRs is impeded by the following challenges.
\begin{itemize}
    \item \textbf{\textit{(C1) Accuracy Challenge.}} One major challenge in deploying INRs using MLP architectures is their inherent \emph{spectral bias}~\cite{rahaman2019spectral,yuce2022structured}. This bias leads neural networks to preferentially learn low-frequency components of functions, resulting in smoother approximations that may lack high-frequency details crucial for accurate representations. To combat this issue, various strategies have been introduced, such as embedding inputs with multiple orthogonal Fourier or Wavelet bases as part of positional encoding~\cite{tancik2020fourier,saragadam2023wire}, employing special periodic or non-periodic activation functions~\cite{sitzmann2020implicit, ramasinghe2022beyond}, and applying learning-based regularization techniques~\cite{li2023regularize, krishnapriyan2021characterizing}. 
    \item \textbf{\textit{(C2) Training Challenge.}} Prior work has addressed training-dynamics issues (e.g., stability) via meta-learned initialization to curb vanishing/exploding signals~\cite{koneputugodage2025vi,tancik2021learned,saratchandran2024activation}, normalization to stabilize and speed training through spectral conditioning~\cite{cai2024batch,bjorck2018understanding}, gradient transformations to alleviate spectral bias via reweighting/preconditioning~\cite{shi2024inductive,bjorck2018understanding}, and Sobolev training to match derivatives when available to damp oscillations~\cite{vlassis2021sobolev,yuan2022sobolev}.
\end{itemize}
Among the current literature, the proposed DINR is a plug-and-play component orthogonal to current methods. \textit{One stone two birds}, it improves accuracy and stabilizes training.

\paragraph{Dynamical Neural Networks.} 
Dynamical neural networks (DNNs), also known as \textit{continuous-time neural networks}, describe latent dynamics that evolve continuously over time~\cite{chen2018neural}. In these models, the hidden dynamics are governed by an ordinary differential equation (ODE) parameterized by neural networks, which can be viewed as a continuous-time generalization of residual connections~\cite{he2016deep}. A major advantage of DNNs is their \textit{parameter efficiency} as they share weights across continuous trajectories, allowing compact models to approximate complex mappings~\cite{chen2018neural,grathwohl2018ffjord,massaroli2020dissecting}. They have been widely adopted in generative modeling through continuous normalizing flows (CNFs)~\cite{chen2018neural,grathwohl2018ffjord,onken2021ot}, where kinetic energy-based regularizers are often introduced to promote smooth trajectories and reduce integration cost~\cite{finlay2020train,yang2020potential,vidal2023taming}. Recently, dynamical formulations also have been explored in other architectures, such as Transformers~\cite{kan2025optimal}. However, these studies mainly focus on regularization and control-theoretic perspectives. In contrast, our work investigates how a dynamical formulation can expand the expressive power of INRs and mitigate spectral bias.

\paragraph{Theoretical Analysis of Neural Networks.} The theoretical understanding of deep learning relies on a diverse set of tools to explain the success of modern neural networks~\cite{soudry2018implicit,arora2018stronger,bartlett2017spectrally,raghu2017expressive,lee2018deep}. Among them, two tools are particularly relevant to our study: 1) capacity measures, such as Rademacher complexity, which provide distribution-agnostic generalization bounds, and 2) infinite-width kernel theories, e.g., NTK, that characterize training dynamics.
\begin{itemize}
    \item \textbf{\textit{(T1) Rademacher Complexity.}} Rademacher complexity~\cite{bartlett2002rademacher,mohri2018foundations} is a data-dependent measure of a model’s expressivity by quantifying how well a hypothesis class, i.e., the set of functions representable by the network, can fit random label noise. Unlike distribution-free metrics, such as Vapnik–Chervonenkis dimension~\cite{vapnik1971uniform}, it adapts to empirical data and yields tighter generalization guarantees. This framework has been applied to various models, including INRs~\cite{zhao2024grounding}, computer vision architectures~\cite{yin2019rademacher,galanti2023norm,trauger2024sequence}, and other models~\cite{park2022learning,marion2023generalization,kim2024bounding}, to characterize the relationship between model capacity and generalization. Recent work \cite{hanson2024rademacher} has derived Rademacher complexity bounds for dynamical networks via Chen–Fliess series expansions. However, in contrast to our work, that analysis focuses on mappings from an initial condition to a terminal scalar output in control‑affine dynamical systems, which studies latent feature dynamics across the full coordinate domain of INRs.
    \item \textbf{\textit{(T2) Neural Tangent Kernel.}} The NTK~\cite{jacot2018neural} provides a principled tool to study optimization dynamics. It reveals that in the infinite-width limit, gradient descent on the squared loss corresponds to kernel regression with a fixed kernel, thereby linking neural network training to classical kernel methods. This insight enables formal analysis of convergence and generalization and has been extended to diverse architectures, such as MLPs, CNNs, RNNs, and Transformers~\cite{lee2019wide,arora2019exact,yang2020tensor,hron2020infinite}. In finite-width regimes, the NTK becomes data-dependent and stochastic~\cite{Hanin2020Finite}, while its spectrum determines convergence rates and biases learning toward low-frequency components~\cite{du2018gradient,bordelon2020spectrum}. To the best of our knowledge, the NTK of DNNs has not been systematically characterized, highlighting the novelty of our analysis for DINR.
\end{itemize}

\section{Method}
\label{sec:method}
With DINR, we propose to extend standard INRs by evolving latent features along continuous trajectories defined by a learnable vector field. This dynamic formulation greatly enlarges the representable function space, providing greater expressivity without increasing network depth or parameter count. In the following, we first review conventional INRs, present a motivating example, then formalize the DINR framework.

\subsection{Preliminary: Implicit Neural Representations}
INRs model signals as continuous mappings from coordinates $\mathbf{x} \in \mathbb{R}^{d_x}$ to values $\mathbf{y} \in \mathbb{R}^{d_y}$ via a neural network $\hat{\mathbf{y}}_{\mathrm{INR}}(\mathbf{x}) \approx \mathbf{y}$. 
A typical INR consists of an input embedding $\phi: \mathbb{R}^{d_x} \to \mathbb{R}^{d_z}$, a latent transformation $f(\, \cdot \,; \theta_\mathrm{stat}):\mathbb{R}^{d_z} \to \mathbb{R}^{d_z}$, and an output decoder $\psi(\, \cdot \,; \theta_\mathrm{out}): \mathbb{R}^{d_z} \to \mathbb{R}^{d_y}$:
\begin{equation}\label{eq:inr_discrete}
\mathbf{z}_0 = \phi(\mathbf{x}), \; 
\mathbf{z}_1 = f(\mathbf{z}_0; \theta_\mathrm{stat}), \; 
\hat{\mathbf{y}}_{\mathrm{INR}} = \psi(\mathbf{z}_1; \theta_\mathrm{out}).
\end{equation}
Given a dataset of $N$ coordinate-signal pairs $\{\mathbf{x}^{(i)}, \mathbf{y}^{(i)}\}_{i=1}^N$, the parameters $\theta=\{\theta_\mathrm{stat},\theta_\mathrm{out} \}$ are obtained by solving the following training problem:
\begin{equation}
\label{eq:INR_train}
\min_{\theta} \left[ \frac{1}{N} \sum_{i=1}^N \mathcal{L}_{\text{data}}\left(\hat{\mathbf{y}}^{(i)}_\mathrm{INR} (\theta), \mathbf{y}^{(i)} \right) \right] ,
\end{equation}
where $\mathcal{L}_{\text{data}}$ is a suitable loss function for the task at hand (e.g., the $\ell_2$ loss for image regression).

Standard INRs perform a \emph{static} feedforward transformation of the input, where the latent feature $\mathbf{z}_1$ is obtained in a single forward pass through $f$. As a result, their representational capacity is strictly bounded by the expressivity of $f$’s fixed-depth architecture, while $\phi$ and $\psi$ mainly serve as lightweight input and output mappings. Such static INRs often exhibit a \emph{spectral bias}~\cite{rahaman2019spectral,yuce2022structured}, limiting their ability to capture high-frequency signal components.

\subsection{Motivating Example: Expanding Function Space via Latent Dynamics}\label{sec:motivation}

To understand how introducing latent dynamics can dramatically enhance the expressivity of such networks, consider a simple one-dimensional example. Let
\[
\mathcal{F} = \{ f(x) = a x^2 + b x + c \mid a,b,c \in \mathbb{R} \}
\]
be the class of quadratic functions. A conventional INR restricted to $\mathcal{F}$ can only represent members of this class.

Now, instead of directly approximating the function, suppose we define a latent variable $z(t)$ that evolves according to an ODE whose velocity field $f$ lies in the same class:
\[
\Phi_{\mathrm{dyn}}\left(\mathcal{F}\right)
= 
\left\{ z(T;x) \;\middle|\; 
\frac{dz}{dt} = f(z), \; f \in \mathcal{F}, \; z(0) = x 
\right\},
\]
where $T$ is the terminal time of the dynamics.
Even with $f(z) = z^2$, we obtain the analytic solution
\[
z(t;x) = \frac{x}{1 - t x} = x + t x^2 + t^2 x^3 + t^3 x^4 + \dots,
\]
which expands into an infinite power series. Although the governing function $f$ is merely quadratic, its induced dynamics generate solutions containing arbitrarily high-order terms. Generally, quadratic (Riccati-type) ODEs can yield rational, logarithmic, or trigonometric solutions \cite{ince2012ordinary}, substantially enlarging the attainable function space.

\begin{figure}
    \centering
        \includegraphics[width=0.99\linewidth]{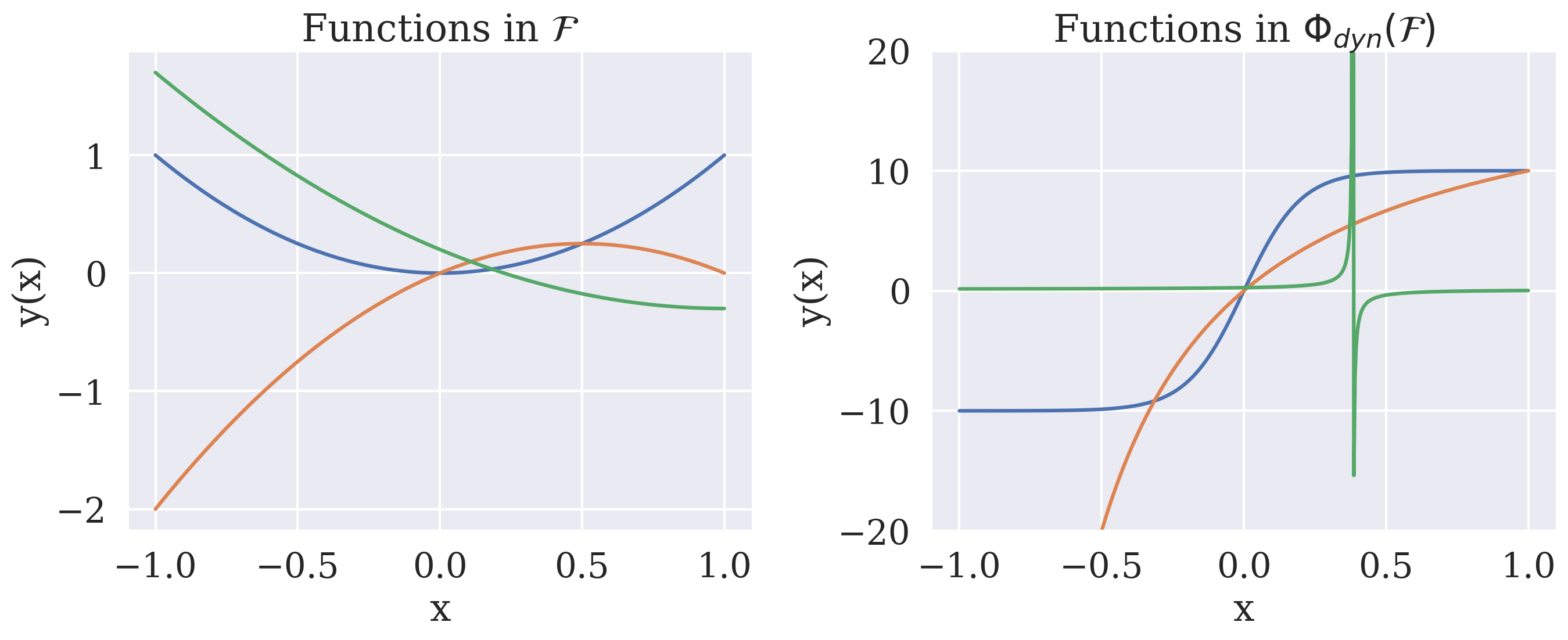}
    \caption{Left: three example functions from the original function class $\mathcal{F}$. 
        Right: three example functions from the expanded function class 
        $\Phi_{\mathrm{dyn}}\left(\mathcal{F}\right)$ obtained using latent dynamics. 
        Different colors indicate distinct functions within each class.}
    \label{fig:motivation_ex}
\end{figure}

This example reveals a key insight. Even when the instantaneous dynamics are simple, their integrated evolution can generate functions of far greater complexity (\autoref{fig:motivation_ex}). Thus, modeling \emph{latent dynamics} provides a principled mechanism to expand the expressive power of a network without increasing its architectural depth or parameter count.

\subsection{Dynamic Implicit Neural Representations}

Motivated by this observation, we propose DINR, a novel formulation that replaces static feedforward computation with the dynamic evolution of latent features. Instead of applying a fixed sequence of discrete transformations as in~\eqref{eq:inr_discrete}, DINR models the hidden state $\mathbf{z}(t)$ as evolving under a learnable vector field:
\begin{equation}
\label{eq:dinr_dynamics}
\frac{d\mathbf{z}(t)}{dt} = f(\mathbf{z}(t), t; \theta_{\mathrm{dyn}}), 
\quad \text{s.t.} \quad 
\mathbf{z}(0) = \phi(\mathbf{x}),
\end{equation}
where $f$ is a neural network parameterized by $\theta_{\mathrm{dyn}}$. The final representation is obtained by integrating this system up to a terminal time $T$ and decoding the terminal state:
\begin{equation}
\label{eq:dinr_output}
\hat{\mathbf{y}} = \psi(\mathbf{z}(T))
= \psi\left(\mathbf{z}(0) + \int_0^T f(\mathbf{z}(t), t; \theta_{\mathrm{dyn}})\, dt \right).
\end{equation}
In practice, the continuous dynamics in~\eqref{eq:dinr_dynamics} is approximated using a numerical ODE solver (e.g., Euler or Runge–Kutta methods), leading to a discrete implementation:
\begin{align}
\mathbf{z}_0 &= \phi(\mathbf{x}), \nonumber \\
\mathbf{z}_{k+1} &= \mathbf{z}_k + \Delta t \cdot f(\mathbf{z}_k, t_k; \theta_{\mathrm{dyn}}), \label{eq:discrete_dyn} \\
\hat{\mathbf{y}}_{\mathrm{DINR}} &= \psi(\mathbf{z}_N; \theta_{\mathrm{out}}) \nonumber
\end{align}
for $k = 0, \dots, N{-}1$. Here, $f$ no longer represents the target function directly, but it defines a \emph{latent feature trajectory} evolving over a latent time variable. By learning the process of feature evolution rather than a single static mapping, DINR achieves a richer function space, smoother information flow, and finer frequency control -- all without increasing network size.
DINR retains the standard input/output networks $(\phi, \psi)$, ensuring compatibility with conventional INRs while introducing a new inductive bias toward dynamic feature representation.

DINR also provides a unifying perspective. While standard INRs correspond to static mappings, DINR generalizes them to dynamic systems whose latent evolution implicitly defines the signal. Even when the instantaneous dynamics $f$ are simple, their accumulated effect yields a significantly richer function space. This dynamic viewpoint not only mitigates spectral bias but also promotes smoother feature evolution, better generalization, and parameter efficiency. In the following sections, we formalize this intuition by analyzing the representational capacity and theoretical properties of the proposed framework.

\subsection{Kinetic Energy Regularization}
To encourage smoother hidden state dynamics, we introduce a kinetic energy (KE) regularizer:
\begin{equation}\label{eq:kinetic_energy}
\mathcal{L}_{\mathrm{KE}} = \sum_{k=0}^{N-1} \|f(\mathbf{z}_k, t_k)\|_2^2 \Delta t.
\end{equation}
In physics, this term corresponds to the kinetic energy of the hidden state trajectory, promoting straighter, non-crossing solution trajectories and more constant velocity over time.
This penalty discourages unnecessarily complex or circuitous trajectories in the feature space, which empirically improves convergence and helps preserve high-frequency details in the output signal (see \Cref{sec:experiments}). In practice, this regularizer is computed alongside the numerical ODE solver~\eqref{eq:discrete_dyn}. Thus, it incurs negligible additional overhead during training.

Overall, given a dataset of $N$ coordinate-signal pairs $\{\mathbf{x}^{(i)}, \mathbf{y}^{(i)}\}_{i=1}^N$, the parameters $\theta=\{\theta_{\mathrm{dyn}}, \theta_{\mathrm{out}} \}$ of our DINR model are obtained by solving the training problem
\begin{equation}
\label{eq:training-objective}
\min_{\theta} \left[ \frac{1}{N} \sum_{i=1}^{N} \mathcal{L}_{\text{data}}\left(\hat{\mathbf{y}}^{(i)}_\mathrm{DINR}(\theta), \mathbf{y}^{(i)}\right) + \mathcal{L}_{\mathrm{KE}}(\theta) \right],
\end{equation}
where $\mathcal{L}_{\text{data}}$ is the same as in~\eqref{eq:INR_train}.
\section{Theoretical Analyses}
\label{sec:theory}

Building on the DINR framework introduced in Section~\ref{sec:method}, we now formalize why latent dynamics enhance both expressivity and trainability. Intuitively, these benefits arise from replacing a single static transformation~\eqref{eq:inr_discrete} with a continuous sequence of latent updates~\eqref{eq:discrete_dyn}, allowing for the emergence of complex functions and richer gradient patterns. We make this precise through analyses of expressivity via Rademacher complexity and trainability via the NTK.

\subsection{Expressivity via Rademacher Complexity}\label{sec:rademacher_main}

Standard INRs compute a latent feature $\mathbf{z}_1 = f(\phi(\mathbf{x}))$ in a single forward pass, which constrains the range of functions they can represent. In contrast, DINR evolves the latent feature $\mathbf{z}(t)$ along a trajectory under a learnable vector field $f$, effectively composing multiple incremental transformations. Even simple dynamics $f$ can generate highly complex mappings when integrated over time as illustrated in Section~\ref{sec:motivation}.

Formally, the Rademacher complexity $\mathcal{R}_n(\mathcal{F})$ quantifies the capacity of a function class $\mathcal{F}$ to represent diverse functions, where $n$ denotes the number of training samples. For DINR, the recursive evolution of latent features substantially enlarges the hypothesis space relative to a static INR:

\begin{theorem}[Expressivity, Informal]\label{thm:rademacher_main}
Let $\mathcal{F}_\mathrm{INR}$ denote the function class of a standard INR, and $\mathcal{F}_\mathrm{DINR}$ mark the function class induced by DINR with integration time $T$. Then,
\[
\mathcal{R}_n(\mathcal{F}_\mathrm{DINR}) \gg \mathcal{R}_n(\mathcal{F}_\mathrm{INR}).
\]
\end{theorem}
Detailed derivations and formal proofs are provided in Appendix~\ref{appen:rademacher}.

This result implies that DINR can represent a broader class of functions, including higher-frequency components, without increasing network size. As empirically demonstrated in Section \ref{sec:experiments}, DINR can achieve superior approximations even with smaller networks. Intuitively, the enhanced expressivity arises from the cumulative effect of incremental latent updates. Each step along the trajectory compounds the influence of the underlying vector field, producing more intricate mappings than a single-step transformation.

\begin{figure*}
\centering
\includegraphics[width=1.0\linewidth]{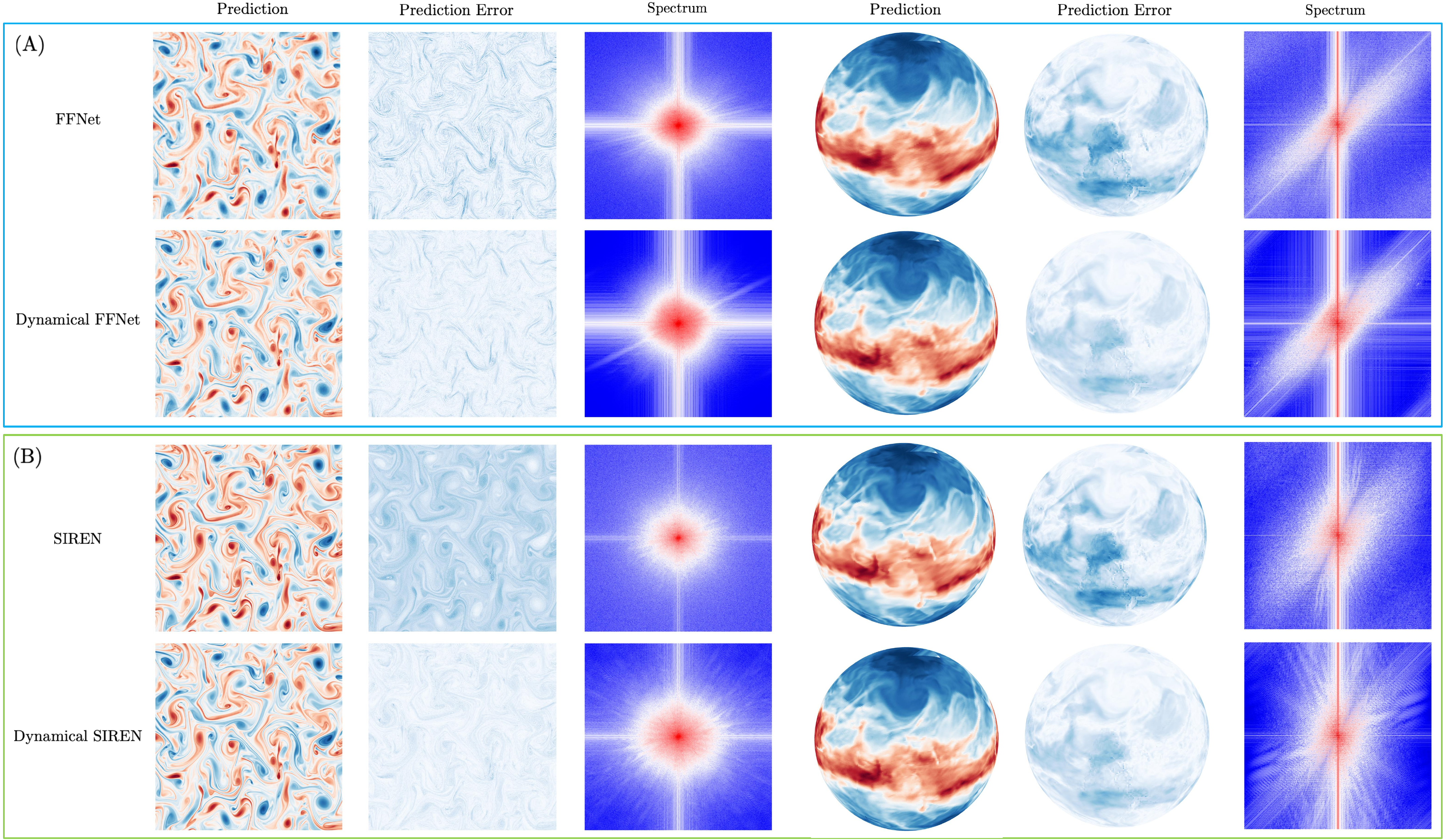}
\caption{Qualitative evaluation of DINR for 2D signal representation. We compare reconstructions from FFNet/SIREN and Dynamical FFNet/SIREN on turbulence and weather data. For each method, we display the reconstructed image, a pixel-wise error map, and the 2D power spectrum of the reconstruction. Error maps are normalized within each model architecture type to highlight the relative distribution of errors; lighter colors indicate lower error.}
\label{fig:qual_1}
\end{figure*}

\paragraph{Regularizing Feature Dynamics.}  
We formally analyze the effect of the kinetic energy regularizer~\eqref{eq:kinetic_energy}, which avoids overly complex latent dynamics while preserving the model’s expressive capacity; seeAppendix~\ref{appen:kinetic_effect_theory} for details.

\subsection{Trainability via Gradient Diversity}\label{sec:ntk_main}

DINR also improves optimization by increasing the diversity of gradients across inputs. Standard INRs rely on a single Jacobian $\partial f / \partial \mathbf{z}_0$, which may be low rank or aligned, limiting effective gradient directions. In DINR, each step along the latent trajectory contributes a local Jacobian $J_k = \partial f / \partial \mathbf{z}_k$, and the overall output gradient is
\[
P = \prod_{k=0}^{N-1} (I + \Delta t\, J_k),
\]
effectively composing multiple local transformations.

\begin{theorem}[Gradient Diversity, Informal]\label{thm:ntk_main}
Unless all local Jacobians $J_k$ are aligned, the product $P$ spans a higher-dimensional space than any individual $J_k$, yielding 
\[
\operatorname{rank}(\Theta_{\mathrm{DINR}}) > \operatorname{rank}(\Theta_{\mathrm{INR}}),
\]
where $\Theta_{\mathrm{DINR}}$ and $\Theta_{\mathrm{INR}}$ denote the NTK of DINR and INR, respectively.
\end{theorem}
Detailed derivations and formal proofs are provided in
Appendix \ref{appen:ntk}.

This theorem indicates that by modeling the \emph{evolution} of latent features rather than static mappings, DINR produces more diverse gradient directions.
This enhanced gradient diversity not only facilitates more stable and accelerated optimization, enabling convergence in fewer training epochs, but also promotes more effective learning of high-frequency components as empirically demonstrated in Section~\ref{sec:experiments}.
\section{Experiments}
\label{sec:experiments}

\subsection{Experimental Setup}
\label{sec:exp_setup}

\paragraph{Datasets.}
We evaluate the proposed dynamical INRs against their static counterparts on four diverse datasets spanning fluid dynamics, weather and climate, and structural biology. Dataset generation and preprocessing details are provided in Appendix~\ref{app:data}.

\begin{itemize}[leftmargin=1.7em,itemsep=0.3ex,topsep=0pt]
    \item \textbf{Turbulent Flow.} 2D Kraichnan turbulence on a doubly periodic square domain $[0,2\pi]^2$, simulated on a $1024 \times 1024$ grid~\cite{holmes2012turbulence,ren2023superbench}.
    \item \textbf{Global Weather.} High-resolution $1440 \times 721$ simulated total column water vapor data from the ERA5 reanalysis dataset~\cite{hersbach2020era5}.
    \item \textbf{Cloud-topped Boundary Layer.} A cloud-resolving simulation of a boundary layer, performed with the UCLA-LES model with resolution at $384 \times 384 \times 130$ and forced by large-scale information from the COSMO-DE model~\cite{klamt2011cosmo,boucher2013clouds}.
    \item \textbf{Cryogenic Electron Microscopy (Cryo-EM).} The EMD-32218 single-particle cryo-EM map of resolution $480 \times 480 \times 480$~\cite{gu2022coupling}. 
\end{itemize}

\begin{table*}
\centering
\caption{Quantitative comparison of DINR and static INR baselines across all tasks.}
\resizebox{2\columnwidth}{!}{
\begin{tabular}{l c c c c c c c c c c c c}
\hline 
\multicolumn{1}{c}{} & \multicolumn{3}{c}{ Turbulence } & \multicolumn{3}{c}{ Weather } & \multicolumn{3}{c}{ Cloud } & \multicolumn{3}{c}{ Cryo-EM }\\
\cmidrule(rl){2-4} \cmidrule(rl){5-7} \cmidrule(rl){8-10} \cmidrule(rl){11-13}
Model & MSE $\downarrow$ & PSNR $\uparrow$ & SSIM $\uparrow$ & MSE $\downarrow$ & PSNR $\uparrow$ & SSIM $\uparrow$ & MSE $\downarrow$ & PSNR $\uparrow$ & SSIM $\uparrow$ & MSE $\downarrow$ & PSNR $\uparrow$ & SSIM $\uparrow$ \\
\hline
FFNet & 9.443e-4 & 30.248 & 0.825 & 1.056e-2 & 19.760 & 0.775 & 9.887e-3 & 20.049 & 0.438 & 8.583e-3 & 20.663 & 0.907 \\
\textit{Dynamical} FFNet & \textbf{4.909e-4} & \textbf{33.089} & \textbf{0.890} & \textbf{9.357e-3} & \textbf{20.288} & \textbf{0.795} & \textbf{2.889e-4} & \textbf{35.391} & \textbf{0.951} & \textbf{2.303e-3} & \textbf{36.376} & \textbf{0.980} \\
Percentage gain & 48.008\% & 9.391\% & 7.796\% & 11.438\% & 2.670\% & 2.567\% & 97.078\% & 76.526\% & 116.724\% & 97.317\% & 76.043\% & 8.123\% \\
\hline
SIREN & 1.731e-3 & 27.615 & 0.961 & 9.192e-3 & 20.365 & 0.791 & 4.280e-2 & 13.684 & 0.770 & 5.215e-2 & 12.827 & 0.533 \\
\textit{Dynamical} SIREN & \textbf{1.959e-4} & \textbf{37.078} & \textbf{0.968} & \textbf{8.844e-3} & \textbf{20.533} & \textbf{0.805} & \textbf{1.297e-3} & \textbf{28.868} & \textbf{0.891} & \textbf{2.362e-2} & \textbf{16.266} & \textbf{0.582} \\
Percentage gain & 88.685\% & 34.268\% & 0.681\% & 3.785\% & 0.823\% & 1.763\% & 96.969\% & 110.957\% & 15.730\% & 54.703\% & 26.812\% & 9.127\% \\
\hline
\end{tabular}
}
\label{tab:quan}
\end{table*}

\noindent \paragraph{Baselines.} Our proposed dynamical formulation is model-agnostic. With minimal code changes, a static INR can be converted into its dynamical variant. We instantiate two widely used INRs as baselines and as DINR backbones: FFNet~\cite{tancik2020fourier} and SIREN~\cite{sitzmann2020implicit}. We create dynamical counterparts called Dynamical FFNet and Dynamical SIREN. In both cases, we replace the discrete stack of layers with a continuous-time feature trajectory $z(t)$ parameterized by an MLP: rectified linear unit (ReLU) activations for FFNet and sine activations for SIREN, respectively.

\noindent \paragraph{Evaluation Protocol.} All experiments are conducted on a single NVIDIA RTX~A6000 (48\,GB). Unless otherwise stated, we train all methods with Adam (initial learning rate $\eta{=}10^{-3}$) under identical optimization settings~\cite{kingma2014adam}. Because model sizes differ, we adjust batch sizes to fit GPU memory (per-dataset values are reported in Appendix~\ref{app:data}). We report mean squared error (MSE), peak signal-to-noise ratio (PSNR, dB), and structural similarity (SSIM).

\subsection{Main Results}
\label{sec:results_main}

\paragraph{Task 1: Data Compression.} A primary application of INRs is the compression of large-scale scientific data, where storage and transmission costs are a significant bottleneck. Fields, e.g., climate modeling, generate petabytes of data from high-resolution simulations, making efficient representation critical. We first evaluate DINR's ability to faithfully represent complex 2D signals compared to conventional static INRs, a core requirement for effective compression. The quantitative and qualitative results are summarized in \autoref{tab:quan} and \autoref{fig:qual_1}, respectively. Key observations include:

\begin{enumerate}[leftmargin=1.7em,itemsep=0.3ex,topsep=0pt]
    \item \textbf{Quantitative Results.} Our quantitative results demonstrate that DINR consistently outperforms its static counterparts across all datasets. For example, improvements are observed in the turbulence dataset, where Dynamical FFNet and Dynamical SIREN achieve MSE reductions of 48.0\% and 88.7\%, respectively, relative to their static baselines. This improved performance underscores the enhanced representational capacity afforded by modeling feature evolution as a continuous-time dynamical system.
    \item \textbf{Qualitative Results.} Spectral analysis, presented in \autoref{fig:qual_1}, provides insight about \textit{why} DINR excels. The 2D power spectra (columns 3 and 6) reveal that while all models effectively capture low-frequency components (the bright central region), static INRs exhibit a pronounced drop off in the high-frequency regions, a clear manifestation of the spectral bias. In stark contrast, DINR models maintain significantly more energy in the outer regions of the spectrum, indicating a superior ability to represent high-frequency details.
    \item \textbf{Compression Results.} Beyond signal fidelity, we also compare model parameters, memory footprint, and the resulting compression ratios on the 2D turbulence dataset. \autoref{tab:compression} shows our dynamical models are significantly more parameter-efficient, achieving a higher compression ratio with a smaller model size. This efficiency stems from the inherent weight sharing across the continuous depth of the DINR formulation, leading to a more compact yet powerful representation.
\end{enumerate}

These results validate Theorems~\ref{sec:rademacher_main} and~\ref{sec:ntk_main}, showing that DINR’s superior high-frequency recovery and accuracy stem from its enhanced expressivity and gradient diversity.
\begin{table}
\centering
\caption{Parameter efficiency and compression ratio.}
\resizebox{\columnwidth}{!}{
\begin{tabular}{l c c c}
\hline 
Model & Trainable Params. & Model Size & Compression Ratio \\
\hline
\multicolumn{4}{c}{ \textit{Turbulence Data} }\\
FFNet & 727 K & 8.4 M & 1.442 \\
\textit{Dynamical} FFNet & \textbf{465 K} & \textbf{5.4 M} & \textbf{2.255} \\
SIREN & 329 K & 3.9 M & 3.187 \\
\textit{Dynamical} SIREN & \textbf{262 K} & \textbf{3.1 M} & \textbf{4.002} \\
\hline
\multicolumn{4}{c}{ \textit{Weather Data} } \\
FFNet & 727 K & 8.4 M & 1.442 \\
\textit{Dynamical} FFNet & \textbf{594 K} & \textbf{6.9 M} & \textbf{1.755} \\
SIREN & 291 K & 3.4 M & 3.603 \\
\textit{Dynamical} SIREN & \textbf{188 K} & \textbf{2.2 M} & \textbf{5.553} \\
\hline
\end{tabular}
}
\label{tab:compression}
\end{table}

\paragraph{Task 2: Field Reconstruction.} A critical challenge in scientific computing and vision is reconstructing a continuous field from sparse, often noisy point measurements. This task moves beyond simple image representation to test a model's ability to learn a resolution-agnostic function that can generalize across the domain, a core problem in fields such as Cryo-EM and geophysical fluid dynamics. We evaluate DINR in this challenging regime, using it to reconstruct 3D volumes from a limited set of irregular samples. The quantitative and qualitative results are presented in \autoref{tab:quan} and \autoref{fig:qual_2}, respectively.

\begin{enumerate}[leftmargin=1.7em,itemsep=0.3ex,topsep=0pt]
    \item \textbf{Quantitative Results.} Results show that dynamical formulation provides a substantial and architecture-agnostic performance leap. On the complex Cloud dataset, DINR achieves a remarkable improvement. The MSE for Dynamical FFNet is more than 34 times lower than its static counterpart (a 97.1\% reduction), corresponding to a massive $+15.34$ dB gain in peak signal-to-noise ratio (PSNR). A similarly dramatic trend is observed for Dynamical SIREN, which reduces MSE by 97.0\% and increases PSNR by $+15.18$ dB.
    \item \textbf{Qualitative Results.} Qualitative analysis in \Cref{fig:qual_2,fig:qual_3} reveals the reasons for this quantitative dominance. For the Cloud data, the standard FFNet captures the low-frequency global structure but fails to resolve the fine, wispy cloud tendril, while Dynamical FFNet accurately synthesizes these high-frequency components, affording a visually faithful reconstruction. Furthermore, on noisy Cryo-EM data, the baseline FFNet is susceptible to overfitting the noisy samples, producing a reconstruction corrupted by spurious, high-frequency artifacts. However, the Dynamical FFNet generates a much cleaner, more coherent structure, suggesting that its continuous dynamics provide implicit regularization and enhance robustness to noisy inputs.
    \item \textbf{Training Dynamics.} Our experiments also highlight DINR's enhanced training stability compared to the standard SIREN, which is highly sensitive to the initialization parameter $\omega_0$. While sharing the same initialization, the baseline SIREN either reconstructed a Cloud signal with an incorrect value range or failed to converge on the Cryo-EM task. In contrast, DINR converged to a high-quality solution in both scenarios. 
\end{enumerate}

The results support Theorems~\ref{thm:rademacher_main} and~\ref{thm:ntk_main}, showing that DINR's enhanced expressivity and gradient diversity improve high-frequency reconstruction accuracy and stability.
\begin{figure}
\centering
\includegraphics[width=1.0\linewidth]{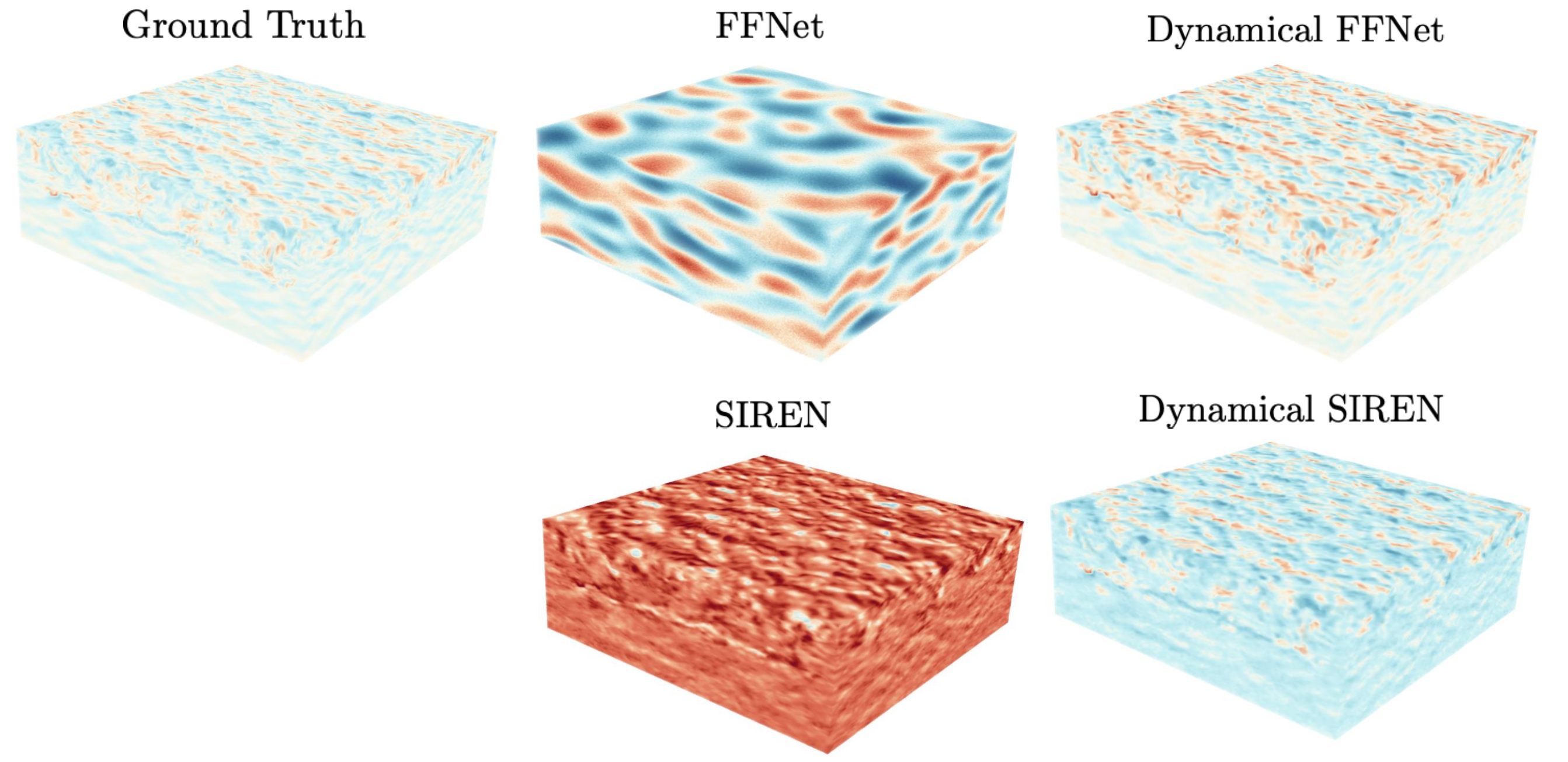}
\caption{Qualitative results of INRs and DINRs on cloud data.}
\label{fig:qual_2}
\end{figure}

\begin{figure}
\centering
\includegraphics[width=1.0\linewidth]{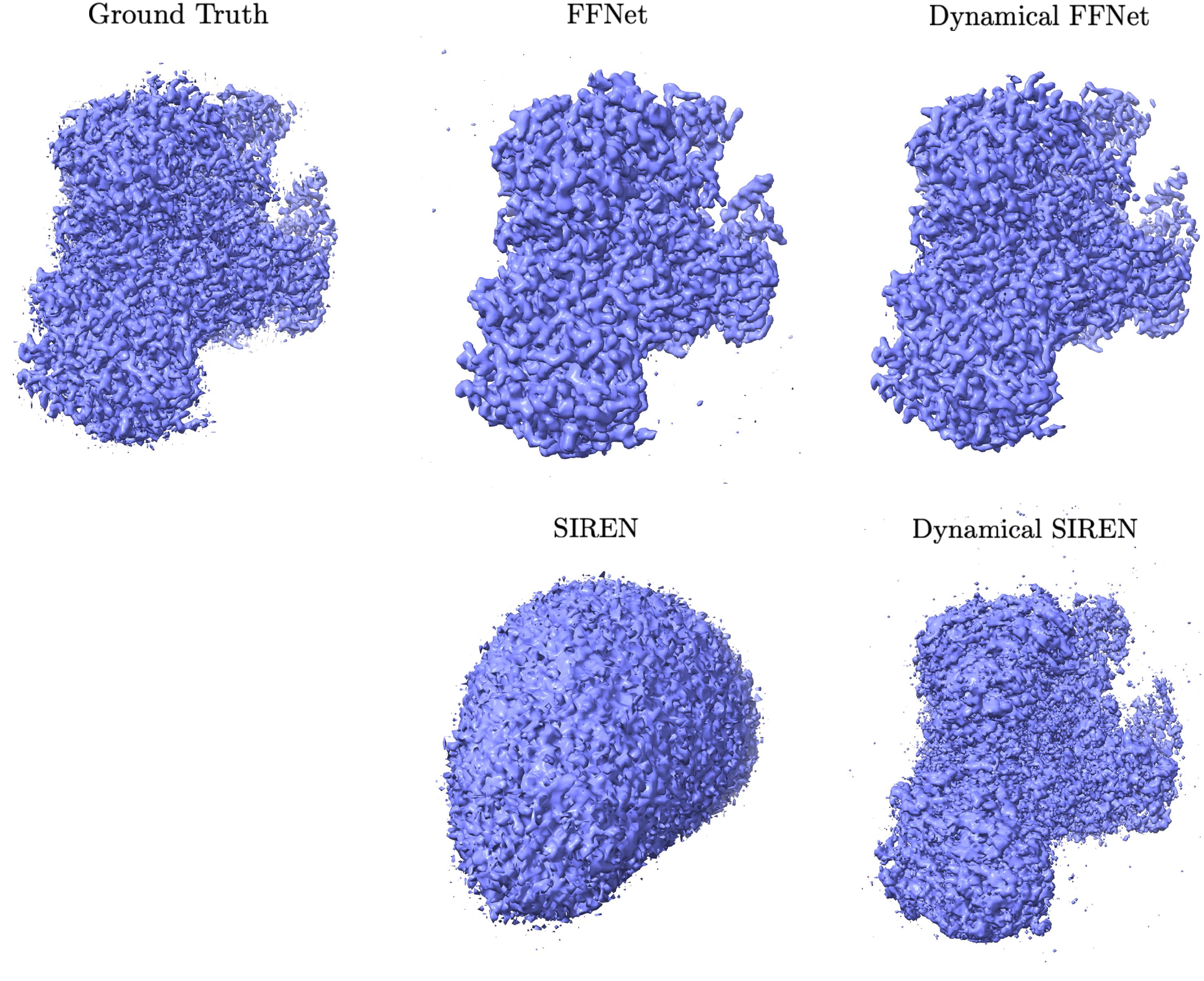}
\caption{Qualitative results of INRs and DINRs on Cryo-EM data.}
\label{fig:qual_3}
\end{figure}

\subsection{Model Analysis and Ablation Results}
\label{sec:results_analysis}

\paragraph{Task 3: Neural Tangent Kernel Analysis.} To gain a theoretical understanding of DINR's advantages, we analyze its function space and optimization landscape via the empirical NTK, which characterizes training behavior in the infinite-width limit and reveals inductive biases. 
Following standard practice, we construct the empirical NTK $K_{ij} = \sum \left\langle \frac{\partial \hat{y}(x_i)}{\partial \theta}, \; \frac{\partial \hat{y}(x_j)}{\partial \theta} \right\rangle$ for each model using a random subset of coordinates and examine its eigenspectrum through: 
1) the effective rank that measures the spectral complexity and the model's capacity to learn diverse functions; 2) the condition number that indicates the optimization problem's stability; and 3) the decay rate of the leading eigenvalues, which reflects how the model prioritizes different functional modes. Our analysis reveals that DINRs consistently induce a more favorable kernel structure on the tested datasets. As summarized in~\Cref{fig:task3}, the dynamical models exhibit a broader, better-conditioned spectrum, characterized by a \textit{higher} effective rank, \textit{lower} condition number, and \textit{slower} eigenvalue decay, supporting Theorem \ref{thm:ntk_main}. Specifically:

\begin{figure*}
\begin{minipage}[t]{0.59\linewidth}
  \centering
  \includegraphics[width=\linewidth]{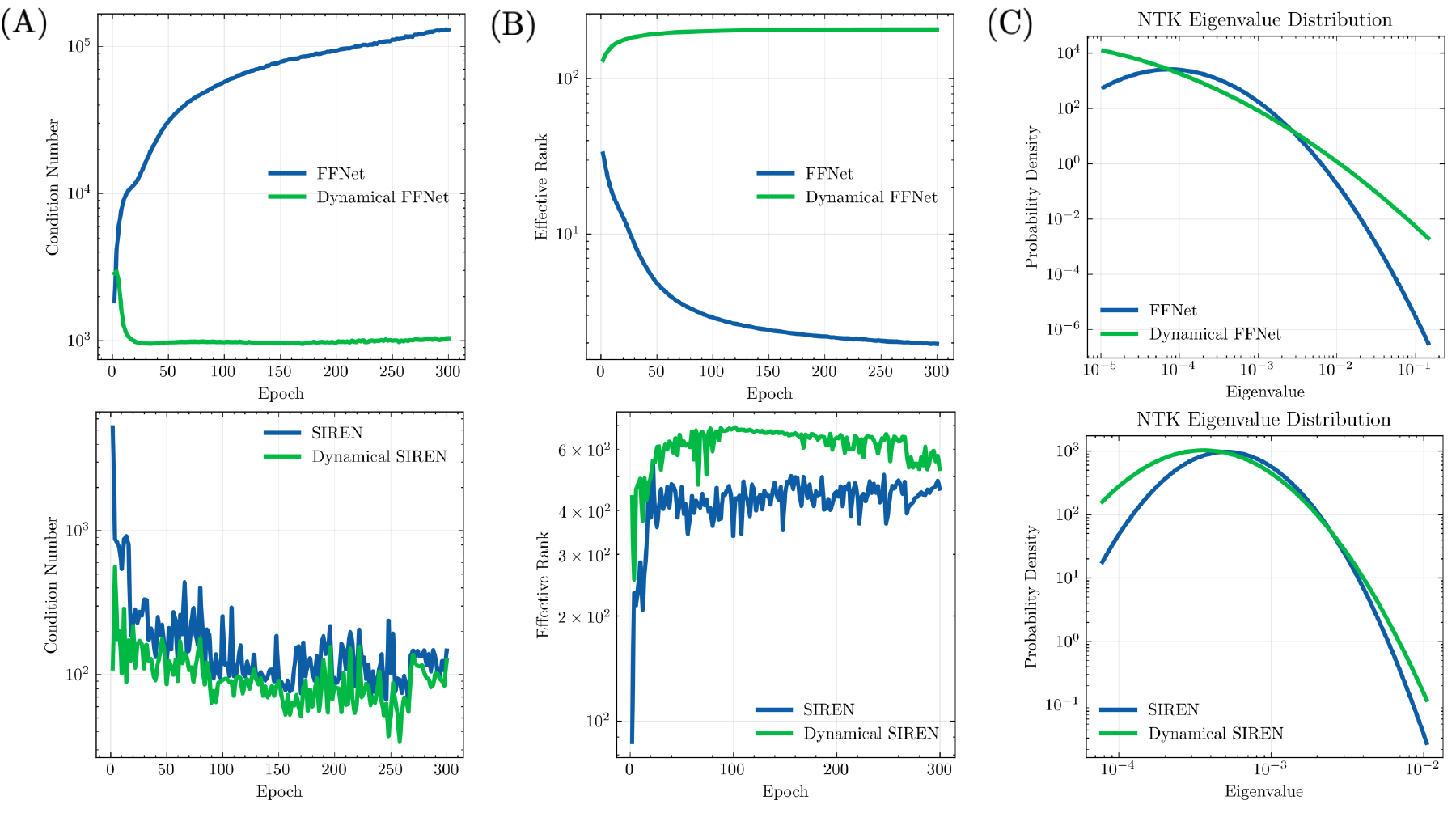}
  \captionof{figure}{Task 3 Results Summary: (A) Effective rank of the empirical NTK during training. (B) Condition number of the empirical NTK during training. (C) Decay rate of leading eigenvalues.}
  \label{fig:task3}
\end{minipage}\hfill
\begin{minipage}[t]{0.39\linewidth}
  \centering
  \includegraphics[width=\linewidth]{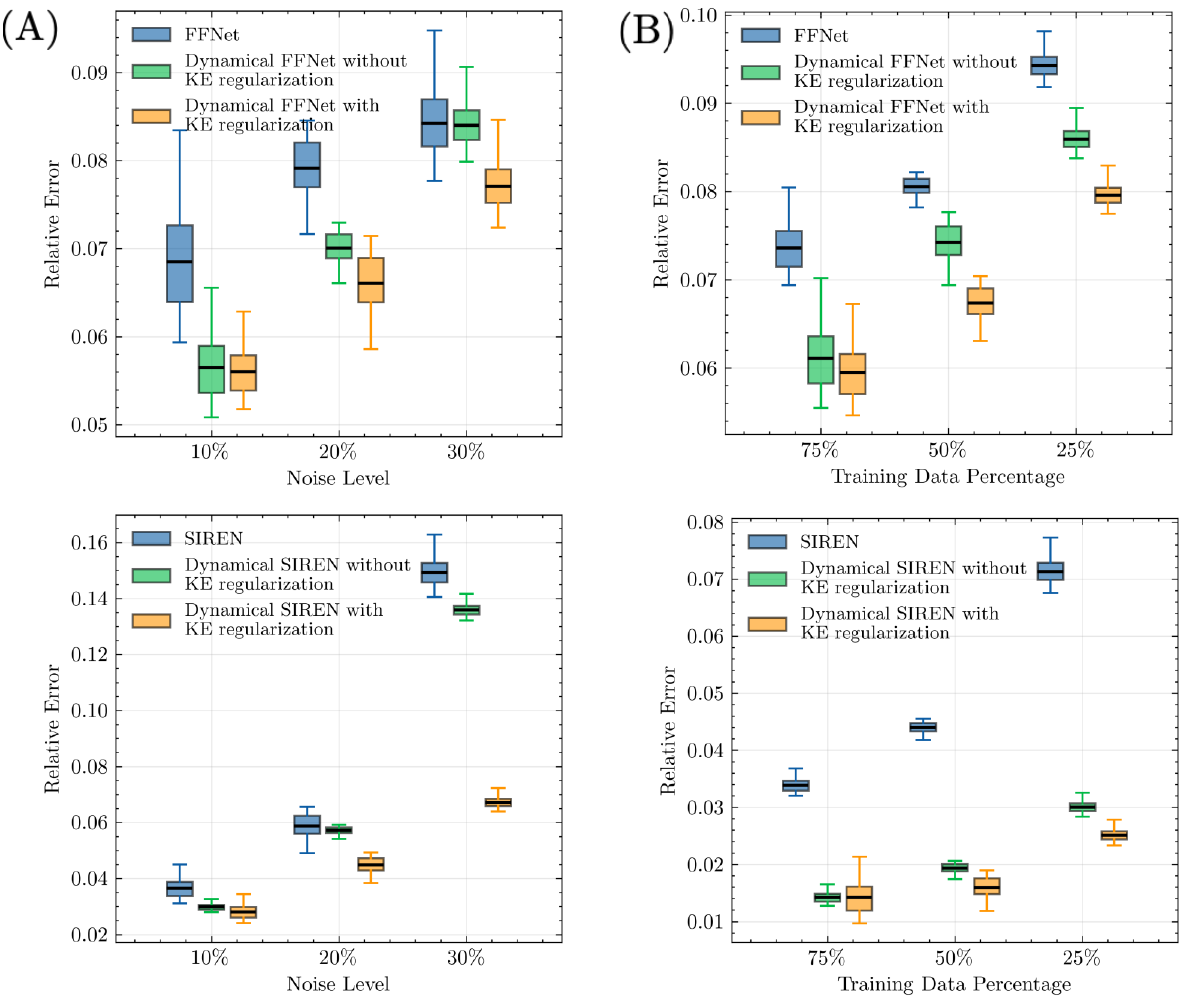}
  \captionof{figure}{Task 4 Results Summary: (A) Performance comparison on noisy turbulence signals and (B) performance comparison under reduced training data.}
  \label{fig:task4}
\end{minipage}
\end{figure*}

\begin{enumerate}[leftmargin=1.7em,itemsep=0.3ex,topsep=0pt]
    \item A \textit{higher effective rank} indicates that DINR's feature space is richer and more expressive. The model can learn a more diverse set of functions, which allows it to capture the complex, high-frequency details that static INRs miss due to spectral bias. This directly explains the superior performance seen in Tasks 1 and 2.
    \item The \textit{lower condition number} implies a better-conditioned optimization landscape, leading to more stable and faster convergence and explaining why Dynamical SIREN succeeds when standard SIREN fails (Task 2) and why DINRs generally train more robustly.
    \item The \textit{slower decay of leading eigenvalues} suggests that DINR distributes its representational capacity more evenly across different modes in the data rather than focusing solely on the dominant, low-frequency components. To better evaluate the eigenvalue statistics, we fit a parametric model to the empirical samples for visualization. Specifically, we use a log-normal distribution. The resulting fits are shown in Figure~\ref{fig:task3}(C). 
\end{enumerate}

\paragraph{Task 4: Ablation Study on the KE Regularizer.}
We conduct a targeted ablation study to validate the effectiveness of our KE regularizer \eqref{eq:kinetic_energy}. By penalizing the kinetic energy of learned trajectories, the regularizer directly encourages smoother latent dynamics and stabilizes the underlying ODE, enhancing robustness and data efficiency.
We evaluate this in two challenging scenarios: reconstruction from noisy data and generalization from scarce data.

\paragraph{Task 4.1 Robustness to Noise.} First, we test the model's ability to reconstruct the turbulence signal from observations corrupted by additive independent and identically distributed (i.i.d.) noise. As shown in~\Cref{fig:task4}, the KE regularizer consistently and significantly reduces reconstruction error across all noise levels.
\begin{enumerate}[leftmargin=1.7em,itemsep=0.3ex,topsep=0pt]
    \item For the FFNet backbone, the regularizer reduces MSE by up to 18.3\% compared to the static baseline. More importantly, its benefit over the unregularized dynamical model grows as noise intensifies, indicating its role in actively combating noise.
    \item The effect is even more dramatic for the expressive SIREN backbone. The KE regularizer slashes MSE by a remarkable 55.0\% at 30\% noise compared to the static baseline. Crucially, while the error of the unregularized dynamical SIREN explodes under heavy noise, the KE-regularized version maintains a much flatter error curve. This demonstrates that the regularizer successfully damps the model's tendency to fit spurious high-frequency components introduced by noise, promoting a smoother, more plausible underlying function.
\end{enumerate}
\paragraph{Task 4.2 Generalization from Scarce Data.} Next, we evaluate the regularizer's impact on generalization when trained on a reduced number of samples. This tests the model's ability to learn a good inductive bias for interpolation and extrapolation.
\begin{enumerate}[leftmargin=1.7em,itemsep=0.3ex,topsep=0pt]
    \item The KE regularizer again provides substantial gains. For the SIREN backbone, the improvement is stark with MSE reductions of up to 64.8\% when trained on only 25\% of the available data.
    \item This result shows the KE term acts as a powerful inductive bias, guiding the model toward smoother solutions that generalize better to unseen regions of the domain. By discouraging overly complex feature trajectories, the regularizer prevents the model from overfitting, a critical property for real-world scenarios where data are often limited.
\end{enumerate}
\section{Conclusion}
\label{sec:conclusion}


We introduced DINR, \emph{Dynamical Implicit Neural Representations}, a framework that models latent features as a continuous-time dynamical system. By evolving features via incremental updates, DINR enhances the expressivity of compact networks, improves trainability, and enables high-fidelity reconstruction of fine details. Our theoretical analysis demonstrates this stepwise evolution boosts representational power and gradient diversity, explaining the empirical gains. This work provides a principled approach to network design that is valuable for applications where parameter efficiency and signal fidelity are critical. Future directions include exploring adaptive step-size or higher-order ODE solvers to optimize the trade-off between expressivity and computational cost.

%

\section*{Acknowledgment}
\label{sec:ack}
This work was supported by the U.S. Department of Energy (DOE), Office of Science, Advanced Scientific Computing Research program under award B\&R-KJ0401010 and FWP-CC147. The UCLA team was also supported in part by DARPA under grant HR00112590074.
{
    \small
    \bibliographystyle{ieeenat_fullname}
    \bibliography{main}
}

\appendix
\clearpage
\setcounter{page}{1}
\maketitlesupplementaryonecol

\section{Data}
\label{app:data}

\subsection{Turbulent Flow}
\label{app:dataset_turbulence}

This dataset features a two-dimensional simulation of Kraichnan turbulence, a canonical problem in fluid dynamics characterized by a wide cascade of energy across a broad range of spatial frequencies. The resulting flow fields are chaotic and contain intricate, high-frequency structures, making this an ideal benchmark for evaluating the ability of implicit neural representations (INRs) to capture complex, multiscale signals. The data are generated via a direct numerical simulation (DNS) of the incompressible Navier-Stokes equations, which govern the motion of viscous fluid substances~\citep{holmes2012turbulence}. The equations enforce the conservation of mass (incompressibility) and momentum:
\begin{equation*}
\nabla \cdot \mathbf{u}=0, \quad \frac{\partial \mathbf{u}}{\partial t}+\mathbf{u} \cdot \nabla \mathbf{u}=-\frac{1}{\rho} \nabla \mathbf{p}+\nu \nabla^2 \mathbf{u},
\end{equation*}
where $\mathbf{u}$ is the velocity field, $\mathbf{p}$ is the pressure, $\rho$ is the fluid density, and $\nu$ is the kinematic viscosity. The simulation is performed in a doubly periodic square domain defined by $[0, 2\pi]^2$ and is discretized on a high-resolution $2048 \times 2048$ grid. The numerical integration employs a second-order energy-conserving Arakawa scheme for the nonlinear Jacobian term~\citep{arakawa1997computational} and a second-order finite-difference scheme for the Laplacian of the vorticity~\citep{ren2023superbench}. For our experiments, we use a single snapshot of the resulting vorticity field, which is normalized to the range $[-1, 1]$ to serve as the ground truth signal.

\subsection{Global Weather Pattern Dataset}
\label{app:dataset_weather}

The global weather pattern data used in our study are derived from the ERA5 atmospheric reanalysis dataset, a state-of-the-art global climate product~\citep{hersbach2020era5}. ERA5 provides a physically consistent and comprehensive record of the Earth's atmosphere by assimilating vast quantities of historical observations from sources including satellites and ground stations into an advanced numerical model. The resulting data fields are characterized by complex, multiscale spatial patterns that span from large-scale circulation features to fine-scale local variations, making it a challenging benchmark for generative and representational models. The dataset offers global coverage with a native spatial resolution of 0.25 degrees, which translates to a Cartesian grid of $721 \times 1440$ pixels. While ERA5 provides a long-term, hourly record of numerous atmospheric variables across 37 vertical pressure levels, our experiments focus on a single, representative snapshot of the total column water vapor (TCWV). This specific variable is chosen for its intricate and detailed structures. For our experimental setup, the scalar TCWV field is extracted and normalized to the range $[0, 1]$ to serve as the ground truth signal.

\subsection{Cloud-topped Boundary Layer}
\label{app:dataset_cloud}

This dataset comprises high-resolution, three-dimensional simulation data of a cloud-topped boundary layer, a critical component of the Earth's climate system that plays a significant role in cloud formation and radiative transfer. The data was generated using the University of California, Los Angeles Large-Eddy Simulation (UCLA-LES) model, a widely recognized tool for simulating atmospheric turbulence and cloud dynamics with high fidelity~\cite{boucher2013clouds,gunther2016mcftle}. The simulation captures the complex, multiscale interactions between turbulent air motion, water vapor, and liquid water content, making it an excellent benchmark for evaluating the ability of neural representations to model intricate, high-frequency spatial phenomena. The simulation was configured to model a stratocumulus-topped boundary layer. It operates under double-periodic boundary conditions, meaning the domain is continuous in the horizontal directions, which is a standard setup for idealized atmospheric studies. The model is driven by homogeneous surface forcing and incorporates large-scale meteorological information from the COSMO-DE numerical weather prediction model to ensure physical realism~\cite{klamt2011cosmo,baldauf2011operational}. For our experiments, we employ a snapshot from the simulation, focusing on a scalar field that represents the liquid water content. This field is defined on a three-dimensional Cartesian grid of $384 \times 384 \times 130$ voxels. The data exhibits fine-grained structures, such as wispy cloud tendrils and sharp gradients at cloud boundaries, which present a significant challenge for implicit neural representation models. The scalar values are normalized to the range $[0, 1]$ to serve as the ground truth signal for our evaluation.

\subsection{Cryo-EM}
\label{app:dataset_cryoem}

The cryogenic electron microscopy (cryo-EM) data used in our experiments is the EMD-32218 entry from the Electron Microscopy Data Bank (EMDB)~\cite{gu2022coupling}. This dataset contains the volumetric three-dimensional density map of the matrix arm of mammalian mitochondrial complex I, captured in its deactive state. Complex I is a crucial enzyme in cellular respiration, and understanding its structure is of significant biological importance. The original study provides a structural basis for the regulation of this enzyme's deactive state. The volumetric map was determined using single-particle analysis, a technique that averages thousands of two-dimensional projection images of individual protein particles to reconstruct a high-resolution three-dimensional volume. The sample organism is \textit{Bos taurus} (cattle). The reconstructed map has a reported resolution of 2.5~\AA{} (Angstroms), providing a detailed view of the protein's atomic structure. The data are provided as a three-dimensional volumetric grid with dimensions of $480 \times 480 \times 480$ voxels. Each voxel is isotropic with a spacing of 1.05~\AA{} along each axis. For our experiments, the volumetric data, originally in CCP4 map format, is normalized to the range $[0, 1]$ to serve as the ground truth signal for evaluating the reconstruction fidelity of the INR models.

\section{Rademacher Complexity Analysis}\label{appen:rademacher}
\subsection{Background: Empirical Rademacher Complexity}
We begin by recalling the formal definition of empirical Rademacher complexity used to quantify the expressiveness of a hypothesis class.

The empirical Rademacher complexity is a standard measure of the expressiveness of a hypothesis class with respect to a fixed sample. It quantifies how well functions in the class can fit random noise and serves as a key tool for deriving generalization bounds.

\begin{definition}[Empirical Rademacher Complexity]
Let \( \mathcal{F} \) be a class of real-valued functions defined on \( \mathcal{X} \), and let \( S = \{x_1, \dots, x_n\} \subset \mathcal{X} \) be a fixed sample of size \( n \). The empirical Rademacher complexity of \( \mathcal{F} \) with respect to \( S \) is defined as:
\[
\mathcal{R}_n(\mathcal{F}) := \mathbb{E}_{\boldsymbol{\sigma}} \left[ \sup_{f \in \mathcal{F}} \frac{1}{n} \sum_{i=1}^n \sigma_i f(x_i) \right],
\]
where \( \boldsymbol{\sigma} = (\sigma_1, \dots, \sigma_n) \) are independent and identically distributed (i.i.d.)\ Rademacher variables uniformly sampled from \( \{-1, +1\} \).
\end{definition}

This definition captures the ability of the function class to align with arbitrary binary noise on a given sample. A higher Rademacher complexity indicates greater capacity and potentially greater risk of overfitting.

\subsection{Theorems and Proofs}
This appendix restates the theoretical results on the Rademacher complexity of standard INRs and dynamical implicit neural representations (DINRs) that were presented informally in Section \ref{sec:rademacher_main}, now providing precise assumptions and constants.

We begin by presenting the main result for DINRs in Corollary~\ref{cor:mlp-flow-rademacher}, along with its supporting lemma. These results demonstrate the exponential dependence on trajectory length, arising from the recursive nature of the DINR dynamics.
For comparison, we also provide an upper bound on the Rademacher complexity of standard INRs in Proposition~\ref{prop:mlp-rademacher}.

\paragraph{Remark on Assumptions.}
The assumptions made in this section concern only the Lipschitz continuity of the involved mappings. These assumptions are standard and widely satisfied in practice for neural networks. In particular, conventional networks equipped with typical activation functions, such as rectified linear unit (ReLU), sigmoid, or tanh, are known to be Lipschitz continuous with explicitly computable Lipschitz constants. Therefore, these conditions do not impose overly restrictive constraints and reflect realistic scenarios in modeling INRs.

\begin{lemma}[Lipschitz Continuity of the Discrete Flow Map]
\label{lem:disc_lipschitz_flow}
Consider the discrete dynamical system induced by the OC‑FFN model:
\[
z_{k+1} = z_k + \Delta t \cdot f(z_k, t_k), \quad k = 0, 1, \dots, N-1,
\]
with initial condition \(z_0 = \phi(x)\). Suppose that
\begin{itemize}
  \item \(\phi : \mathbb{R}^{d_x} \to \mathbb{R}^{d_z}\) is \(L_\phi\)-Lipschitz,
  \item For each \(t_k\), \(f(\cdot, t_k)\) is \(L_f\)-Lipschitz in the first argument (uniformly in \(k\)).
\end{itemize}
Then, the discrete-time flow map \(x \mapsto z_N\) is Lipschitz continuous, and one may bound
\[
\|z_N(x_1) - z_N(x_2)\| \;\le\; L_\phi\, (1 + \Delta t\,L_f)^N \;\|x_1 - x_2\|.
\]
In particular, when \(N \Delta t = T\) and \(\Delta t\,L_f\) is small, we have the approximation
\[
(1 + \Delta t\,L_f)^N \approx e^{L_f T},
\]
resulting in  
\[
\|z_N(x_1) - z_N(x_2)\| \;\lesssim\; L_\phi\, e^{L_f T} \, \|x_1 - x_2\|.
\]
\end{lemma}

\begin{proof}
Let \(x_1, x_2 \in \mathcal{X}\). Denote the corresponding trajectories by \(\{z_k^{(1)}\}\) and \(\{z_k^{(2)}\}\), initialized as
\[
z_0^{(i)} = \phi(x_i), \quad i = 1,2.
\]
We prove by induction a one-step Lipschitz inequality. For the base case \(k = 0\), we have
\[
\|z_0^{(1)} - z_0^{(2)}\| = \|\phi(x_1) - \phi(x_2)\| \le L_\phi\,\|x_1 - x_2\|.
\]

Assume for some \(k\ge0\),
\[
\|z_k^{(1)} - z_k^{(2)}\| \le L_\phi\,(1 + \Delta t\,L_f)^k \,\|x_1 - x_2\|.
\]
Then, at the next step,
\[
\begin{aligned}
& \|z_{k+1}^{(1)} - z_{k+1}^{(2)}\| \\
&= \|z_k^{(1)} + \Delta t\,f(z_k^{(1)}, t_k) \;-\; (z_k^{(2)} + \Delta t\, f(z_k^{(2)}, t_k))\| \\
&\le \|z_k^{(1)} - z_k^{(2)}\| + \Delta t\, \|f(z_k^{(1)}, t_k) - f(z_k^{(2)}, t_k)\| \\
&\le \|z_k^{(1)} - z_k^{(2)}\| + \Delta t\,L_f\,\|z_k^{(1)} - z_k^{(2)}\| \\
&= (1 + \Delta t\,L_f)\,\|z_k^{(1)} - z_k^{(2)}\| \\
&\le (1 + \Delta t\,L_f)\, \bigl( L_\phi\,(1 + \Delta t\,L_f)^k \bigr)\, \|x_1 - x_2\| \\
&= L_\phi\,(1 + \Delta t\,L_f)^{k+1} \,\|x_1 - x_2\|.
\end{aligned}
\]

By induction, this inequality holds for all \(k = 0, 1, \dots, N\). In particular, for \(k = N\),
\[
\|z_N^{(1)} - z_N^{(2)}\| \le L_\phi\,(1 + \Delta t\,L_f)^N \,\|x_1 - x_2\|.
\]

This completes the proof.  
\end{proof}

\begin{proposition}[Rademacher Complexity of DINRs]
\label{prop:reachability-covering}
Let \(\mathcal{F}_\mathrm{DINR} = \{ y:\mathbb{R}^{d_x}\to \mathbb{R}^{d_y} \mid y(x) = \psi(z(T; x)) \}\) be a hypothesis class defined on inputs \(x \in \mathcal{X} \subset \mathbb{R}^{d_x}\) and let \(n\) denote the sample size. Suppose:
\begin{itemize}
    \item \(\operatorname{diam}(\mathcal{X}) \le D\),
    \item $\phi : \mathbb{R}^{d_x} \to \mathbb{R}^{d_z}$ is a fixed, $B_\phi$-bounded feature map,
    \item $z(t) \in \mathbb{R}^{d_z}$ solves $\dot{z}(t) = f(z(t), t)$ with $z(0) = \phi(x)$, and the vector field $f(z, t)$ is $L_f$-Lipschitz in $z$,
    \item $\psi : \mathbb{R}^{d_z} \to \mathbb{R}^{d_y}$ is $L_\psi$-Lipschitz.
\end{itemize}
Then, the empirical Rademacher complexity with respect to samples of size \(n\) satisfies
\[
\mathcal{R}_n(\mathcal{F}_\mathrm{DINR}) \le \frac{C \cdot L_{\psi} L_\phi D \cdot e^{L_f T} \cdot \sqrt{md_y}}{\sqrt{n}},
\]
for some universal constant \(C\).
\end{proposition}

\begin{proof}
By Lemma~\ref{lem:disc_lipschitz_flow}, the mapping $x \mapsto z(T; x)$ is Lipschitz with constant at most
\[
L_{\mathrm{flow}} := L_\phi e^{L_f T},
\]
and the reachable set of the dynamics $\mathcal{Z}_T := \{ z(T; x) : x \in \mathcal{X} \}$ admits an $\epsilon$-covering with covering number:
\[
\mathcal{N}(\epsilon, \mathcal{Z}_T, \|\cdot\|) \le \left( \frac{L_{\mathrm{flow}} D}{\epsilon} \right)^m.
\]

Because $\psi : \mathbb{R}^{d_z} \to \mathbb{R}^{d_y}$ is $L_{\psi}$-Lipschitz, the composed function \( y(x) = \psi(z(T; x)) \) is \( L_{\psi} L_{\mathrm{flow}} \)-Lipschitz in \(x\). Thus,
\[
\|y(x_1) - y(x_2)\| \le L_{\psi} L_\phi e^{L_f T} \|x_1 - x_2\| =: L_{\mathrm{comp}} \cdot \|x_1 - x_2\|.
\]

As such, the diameter of the function class \(\mathcal{F}_\mathrm{DINR}\) satisfies:
\[
\operatorname{diam}(\mathcal{F}_\mathrm{DINR}) \le L_{\mathrm{comp}} D =: C_0.
\]

To upper bound the empirical Rademacher complexity of vector-valued function classes \(\mathcal{F}_\mathrm{DINR} \subset \{ f : \mathcal{X} \to \mathbb{R}^{d_y} \}\), we apply a standard vector-valued Dudley entropy integral bound \cite{bartlett2002rademacher, maurer2016vector}:
\[
\mathcal{R}_n(\mathcal{F}_\mathrm{DINR}) \le \frac{12 \sqrt{d_y}}{\sqrt{n}} \int_0^{\operatorname{diam}(\mathcal{F}_\mathrm{DINR})} \sqrt{ \log \mathcal{N}(\epsilon, \mathcal{F}_\mathrm{DINR}, \|\cdot\|) }\, d\epsilon.
\]

From the Lipschitz property of \( y(x) \), the covering number of \(\mathcal{F}_\mathrm{DINR}\) satisfies:
\[
\mathcal{N}(\epsilon, \mathcal{F}_\mathrm{DINR}) \le \left( \frac{L_{\mathrm{comp}} D}{\epsilon} \right)^m.
\]
Taking logarithms:
\[
\log \mathcal{N}(\epsilon, \mathcal{F}_\mathrm{DINR}) \le m \cdot \log\left( \frac{L_{\psi} L_\phi D \cdot e^{L_f T}}{\epsilon} \right) = m \cdot \log\left( \frac{C_0}{\epsilon} \right).
\]

Substituting into the entropy integral:
\[
\mathcal{R}_n(\mathcal{F}_\mathrm{DINR}) \le \frac{12 \sqrt{d_y m}}{\sqrt{n}} \int_0^{C_0} \sqrt{ \log\left( \frac{C_0}{\epsilon} \right) }\, d\epsilon.
\]

Now, use the change of variable \( \epsilon = C_0 e^{-u} \), yielding:
\[
\int_0^{C_0} \sqrt{ \log\left( \frac{C_0}{\epsilon} \right) }\, d\epsilon 
= \int_0^\infty \sqrt{u} \cdot C_0 e^{-u} \, du 
= C_0 \int_0^\infty u^{1/2} e^{-u} \, du
 = C_0 \cdot \Gamma(3/2)
 = C_0 \cdot \frac{\sqrt{\pi}}{2}.
\]

Combining all terms, we obtain:
\[
\mathcal{R}_n(\mathcal{F}_\mathrm{DINR}) \le \frac{12 \sqrt{d_y m}}{\sqrt{n}} \cdot C_0 \cdot \frac{\sqrt{\pi}}{2}
= \frac{C \cdot L_{\psi} L_\phi D \cdot e^{L_f T} \cdot \sqrt{m d_y}}{\sqrt{n}}
\]
for a universal constant \(C = 6\sqrt{\pi}\).
\end{proof}

\begin{corollary}
\label{cor:mlp-flow-rademacher}
Under the same assumptions as Proposition~\ref{prop:reachability-covering}, suppose additionally that the vector field \( f : \mathbb{R}^{d_z} \to \mathbb{R}^{d_z} \) is implemented as a depth-\( \ell \) feedforward neural network of the form
\[
f = f_\ell \circ f_{\ell-1} \circ \cdots \circ f_1,
\]
where each layer \( f_i \) is Lipschitz continuous with constant at most \( L_0 \), i.e.,
\[
\|f_i(u) - f_i(v)\| \le L_0 \|u - v\| \quad \text{for all } u, v \in \mathbb{R}^{d_z}.
\]
Then, the empirical Rademacher complexity of the corresponding hypothesis class satisfies
\[
\mathcal{R}_n(\mathcal{F}_\mathrm{DINR}) \le \frac{C \cdot L_{\psi} L_\phi D \cdot e^{L_0^\ell T} \cdot \sqrt{md_y}}{\sqrt{n}}
\]
for some universal constant \( C \).
\end{corollary}

\begin{proof}
By composition of Lipschitz functions, the network \( f = f_\ell \circ \cdots \circ f_1 \) is \( L_f = L_0^\ell \)-Lipschitz. Substituting this bound into Proposition~\ref{prop:reachability-covering} yields the stated result.
\end{proof}

\begin{proposition}[Rademacher Complexity of INRs]
\label{prop:mlp-rademacher}
Let \(\mathcal{F}_\mathrm{INR} = \{ y:\mathbb{R}^{d_x}\to \mathbb{R}^{d_y} \mid y(x) = \psi(f(\phi(x))) \}\) be a hypothesis class defined on inputs \(x \in \mathcal{X} \subset \mathbb{R}^{d_x}\) and let \(n\) denote the sample size. Suppose:
\begin{itemize}
    \item \(\operatorname{diam}(\mathcal{X}) \le D\),
    \item \(\phi : \mathbb{R}^{d_x} \to \mathbb{R}^{d_z}\) is a fixed, \(B_\phi\)-bounded feature map,
    \item \(f : \mathbb{R}^{d_z} \to \mathbb{R}^{d_z}\) is a depth-\(\ell\) feedforward neural network of the form
    \[
    f = f_\ell \circ f_{\ell-1} \circ \cdots \circ f_1,
    \]
    where each layer \(f_i\) is Lipschitz continuous with constant at most \(L_0\), i.e.,
    \[
    \|f_i(u) - f_i(v)\| \le L_0 \|u - v\| \quad \text{for all } u,v \in \mathbb{R}^{d_z},
    \]
    \item \(\psi : \mathbb{R}^{d_z} \to \mathbb{R}^{d_y}\) is \(L_{\psi}\)-Lipschitz.
\end{itemize}
Then, the empirical Rademacher complexity satisfies:
\[
\mathcal{R}_n(\mathcal{F}_\mathrm{INR}) \le \frac{\tilde{C} \cdot L_{\psi} L_\phi L_0^\ell D \cdot \sqrt{md_y}}{\sqrt{n}}
\]
for some universal constant \(\tilde{C}\).
\end{proposition}

\begin{proof}
By composition of Lipschitz functions, the overall Lipschitz constant of the mapping
\[
x \mapsto \psi(f(\phi(x)))
\]
is at most \(L_{\psi} L_0^\ell L_\phi\). As the input space \(\mathcal{X}\) has diameter at most \(D\), the image of \(\mathcal{X}\) under this composition lies in a subset of \(\mathbb{R}\) with diameter at most \(L_{\psi} L_0^\ell L_\phi D\). 

Applying Dudley's entropy integral bound (as in the proof of Proposition~\ref{prop:reachability-covering}), we obtain:
\[
\mathcal{R}_n(\mathcal{F}_\mathrm{INR}) \le \frac{\tilde{C} \cdot L_{\psi} L_\phi L_0^\ell D \cdot \sqrt{md_y}}{\sqrt{n}}
\]
for some universal constant \(\tilde{C}\) as claimed.
\end{proof}

\subsection{Impact of Kinetic Energy Regularization on Generalization}\label{appen:kinetic_effect_theory}
In this section, we theoretically observe that regularizing the kinetic energy \eqref{eq:kinetic_energy} in DINRs improves their generalization and mitigates overfitting. 
We first recall a standard generalization bound based on the Rademacher complexity \cite{bartlett2002rademacher,mohri2018foundations}.

\begin{theorem}[Generalization Bound via Rademacher Complexity \cite{bartlett2002rademacher,mohri2018foundations}]
\label{thm:gen-bound}
Let \(\mathcal{F}\) be a class of functions mapping from \(\mathcal{X}\) to \([-B, B]\) and let \(\{(x_i, y_i)\}_{i=1}^n\) be i.i.d. samples drawn from a distribution \(\mathcal{D}\). Then, for any \(\delta > 0\), with probability at least \(1 - \delta\), the following holds uniformly for all \(f \in \mathcal{F}\):
\[
\mathbb{E}_{(x,y) \sim \mathcal{D}} \left[(f(x) - y)^2 \right] \leq \frac{1}{n} \sum_{i=1}^n (f(x_i) - y_i)^2 + 4B \cdot \mathcal{R}_n(\mathcal{F})+ 3B^2 \sqrt{\frac{\log(2/\delta)}{2n}},
\]
where \(\mathcal{R}_n(\mathcal{F})\) denotes the empirical Rademacher complexity of \(\mathcal{F}\).
\end{theorem}

This bound implies that a DINR with larger Rademacher complexity has greater expressive power but also higher risk of overfitting. 
In DINRs, the complexity of the latent dynamics is influenced by the length of the trajectory. 
The following proposition shows that kinetic energy regularization constrains this trajectory length and, thus, reduces the Rademacher complexity.

\begin{proposition}[Rademacher Complexity Bound for Regularized Class ]
Let \(\mathcal{F}_C\) be the hypothesis class induced by the OC-regularized architecture with discrete kinetic energy bounded by \(E\):
\[
\mathcal{F}_\mathrm{DINR}^E := \left\{ x \mapsto \psi(z(T)) \;\middle|\; \sum_{k=0}^{N-1} \|f(z_k, t_k; \theta_{\mathrm{dyn}})\|^2 \, \Delta t \le E \right\}.
\]
Assume:
\begin{itemize}
     \item \(\operatorname{diam}(\mathcal{X}) \le D\),
    \item \(\phi : \mathbb{R}^{d_x} \to \mathbb{R}^{d_z}\) is a fixed, \(B_\phi\)-bounded feature map,
     \item \(\psi : \mathbb{R}^{d_z} \to \mathbb{R}^{d_y}\) is \(L_{\psi}\)-Lipschitz.
\end{itemize}
Then, the empirical Rademacher complexity of \(\mathcal{F}_\mathrm{DINR}^E\) satisfies:
\[
\mathcal{R}_n(\mathcal{F}_\mathrm{DINR}^E) \le \frac{C' \cdot L_{\psi} \cdot L_\phi D \cdot \sqrt{m} \cdot \left(B_\phi + \sqrt{T E}\right)}{\sqrt{n}}
\]
for some universal constant \(C'\).
\end{proposition}

\begin{proof}
Starting from the dynamics,
\[
z_{k+1} = z_k + \Delta t \cdot f(z_k, t_k), \quad k = 0, 1, \dots, N-1,
\]
with initial points bounded as \(\|z(0)\| = \|\phi(x)\| \le B_\phi\).

By discrete Cauchy--Schwarz,
\[
\Big\| \sum_{k=0}^{N-1} f(z^{(i)}_k, t_k)\ \Delta t  \Big\| 
\le  \sum_{k=0}^{N-1} \Delta t \, \|f(z^{(i)}_k, t_k)\|
\le  \sqrt{\Big(\sum_{k=0}^{N-1} \Delta t\Big)
          \Big(\sum_{k=0}^{N-1} \|f(z^{(i)}_k, t_k)\|^2 \, \Delta t\Big)} 
\le  \sqrt{T E}.
\]

Hence, the reachable set at time \(T\) is contained in a ball of radius
\[
R := B_\phi + \sqrt{T E}.
\]
Given that \(\phi\) is \(L_\phi\)-Lipschitz and the input space \(\mathcal{X}\) has diameter \(D\), the initial points \(z(0) = \phi(x)\) lie within a set of diameter at most \(L_\phi D\).

The reachable set radius \(R\) bounds the size of the image of \(\mathcal{X}\) under the flow. By standard covering number arguments (see Proposition~\ref{prop:reachability-covering}), the covering number at scale \(\epsilon\) of the output space satisfies
\[
\mathcal{N}(\epsilon) \le \left( \frac{C'' R L_\phi D \sqrt{md_y}}{\epsilon} \right)^m
\]
for some constant \(C''\).
Therefore, we have 
\begin{align*}
\mathcal{R}_n(\mathcal{F}_\mathrm{DINR}^E) & \le \frac{C' \cdot L_{\psi} \cdot L_\phi D \cdot \sqrt{m} \cdot R}{\sqrt{n}} \\
& = \frac{C' \cdot L_{\psi} \cdot L_\phi D \cdot \sqrt{m} \cdot \left(B_\phi + \sqrt{T E}\right)}{\sqrt{n}}.\qedhere
\end{align*}
\end{proof}

This result provides a theoretical justification for the effect of kinetic energy regularization. As the kinetic energy bound \(E\) decreases, the Rademacher complexity is reduced, leading to better generalization as indicated by Theorem~\ref{thm:gen-bound}. 
In practice, training a DINR without kinetic energy regularization can lead to a large effective \(E\), while adding this regularization controls \(E\) and balances expressive power and generalization.

\section{Neural Tangent Kernel (NTK) Analysis}
\label{appen:ntk}

\subsection{Background: NTK and Gradient Dynamics}
We briefly review the NTK and explain how its spectrum and rank influence neural network training dynamics.

\paragraph{Definition of NTK.} 
Let \(\hat{y}(x;\theta)\) be a neural network with parameters \(\theta\). The NTK is defined as
\[
\Theta(x, x') = \nabla_\theta \hat{y}(x)^\top \nabla_\theta \hat{y}(x'), \quad \Theta \in \mathbb{R}^{n \times n},
\]
where \(\Theta_{ij} = \nabla_\theta \hat{y}(x_i)^\top \nabla_\theta \hat{y}(x_j)\) for a dataset \(\{x_i\}_{i=1}^n\). Equivalently, let 
\[
J = \begin{bmatrix}
\nabla_\theta \hat{y}(x_1)^\top \\
\vdots \\
\nabla_\theta \hat{y}(x_n)^\top
\end{bmatrix} \in \mathbb{R}^{n \times P},
\]
then \(\Theta = J J^\top\) and, therefore,
\[
\operatorname{rank}(\Theta) = \operatorname{rank}(J).
\]
The NTK captures how changes in network parameters affect outputs across the dataset. Importantly, the \emph{rank} of \(\Theta\) quantifies the number of independent directions in function space that the network can influence through parameter updates, directly reflecting the network's capacity to learn diverse components of the target function.

\paragraph{Gradient Flow Dynamics.} 
Consider training \(\hat{y}(x;\theta)\) via gradient flow on a dataset \(\{(x_i,y_i)\}_{i=1}^n\) using the mean squared error (MSE) loss
\[
\mathcal{L}(\theta) = \frac{1}{2} \sum_{i=1}^n (\hat{y}(x_i;\theta) - y_i)^2.
\]
The gradient flow dynamics are then
\[
\frac{d}{dt} \theta_t = - \nabla_\theta \mathcal{L}(\theta_t), \quad 
\hat{y}_t = [\hat{y}(x_1;\theta_t), \dots, \hat{y}(x_n;\theta_t)]^\top,
\]
where \(\theta_t\) denotes the network parameters evolving continuously over training time \(t \geq 0\).

In the \emph{NTK regime}, we assume \(\nabla_\theta \hat{y}(x)\) remains approximately constant during training, resulting in the linearized dynamics
\[
\frac{d}{dt} \hat{y}_t = - \Theta (\hat{y}_t - y), \quad y = [y_1, \dots, y_n]^\top.
\]

\paragraph{Spectral Decomposition and Convergence.} 
Because \(\Theta\) is symmetric positive semi-definite, it admits the eigendecomposition \(\Theta = U \Lambda U^\top\) with \(\Lambda = \mathrm{diag}(\lambda_1,\dots,\lambda_n)\) and \(U\) orthonormal. The solution to the linearized dynamics is
\[
\hat{y}_t - y = \sum_{i=1}^n e^{-\lambda_i t} \langle u_i, \hat{y}_0 - y \rangle \, u_i.
\]
Taking norms gives
\[
\|\hat{y}_t - y\|^2 = \sum_{i=1}^{r} e^{-2\lambda_i t} \langle u_i, \hat{y}_0 - y \rangle^2, \quad r = \operatorname{rank}(\Theta).
\]

This decomposition shows that each eigendirection \(u_i\) is learned at a rate determined by its eigenvalue \(\lambda_i\), while directions corresponding to zero eigenvalues are never learned. Thus, the NTK rank directly measures the number of directions in the target function space that the network can effectively learn. Higher NTK rank implies a broader span of learnable directions, allowing the network to fit a more diverse set of signal components simultaneously, which leads to faster and more robust convergence.

In summary, the NTK rank quantifies the effective capacity of the network under gradient descent. In the following, we show the DINR architecture can strictly increase the NTK rank compared to a standard INR, enabling richer and more expressive learning dynamics.

\subsection{Gradient Derivation of INR and DINR}\label{appen:gradients}
Before analyzing the NTK rank, we derive the gradients of the standard INR and the proposed DINR with respect to their parameters. 
This subsection introduces a convenient notation for cumulative Jacobians along the latent trajectory of the DINR, which will be used in subsequent theoretical analysis.
\paragraph{Standard INR.}
Consider first the standard INR defined as:
\[
\hat{y}_{\mathrm{INR}}(x) = \psi\big(f(\phi(x))\big),
\]
where \(\phi\), \(k\), \(f\), and \(\psi\) are as previously defined.

The gradient of the standard INR output with respect to parameters \(\theta = (\theta_{\mathrm{\psi}}, \theta_{\mathrm{f}}, \theta_{\mathrm{\phi}})\) is given by:
\[
\nabla_{\theta} \hat{y}_{\mathrm{INR}}(x) =
\begin{bmatrix}
\frac{\partial \psi}{\partial \theta_{\mathrm{\psi}}}(z_1) \\
\frac{\partial \psi}{\partial z_1} \cdot \frac{\partial f}{\partial \theta_{\mathrm{f}}}(z_0) \\
\frac{\partial \psi}{\partial z_1} \cdot \frac{\partial f}{\partial z_0} \cdot \frac{\partial \phi}{\partial \theta_{\mathrm{\phi}}}(x)
\end{bmatrix},
\]
where \( z_0 := \phi(x) \) and \( z_1 := f(z_0) \).

\paragraph{DINR.}
Now, consider the DINR, where the latent state evolves recursively as
\[
z_{k+1} = z_k + \Delta t \cdot f(z_k, t_k; \theta_{\mathrm{f}}), \quad k=0\dots,N-1,
\]
with initial state \( z_0 = \phi(x; \theta_{\mathrm{\phi}}) \) and output
\[
\hat{y}_\mathrm{DINR}(x) = h(z_N; \theta_{\mathrm{\psi}}).
\]

For notational convenience, define the Jacobian matrices
\[
J_k := \frac{\partial f}{\partial z_k} \in \mathbb{R}^{d \times d}, \quad k=0,\dots,N-1,
\]
and define the partial cumulative Jacobians for any \( a \leq b \) as
\[
P_{a:b} := P_{0:N-1} \prod_{j=a}^{b} \left( I + \Delta t \cdot J_j \right)
\]
with the convention that \( P_{a:b} = I \) if \( a > b \).
The cumulative Jacobian along the latent trajectory is then:
\[
P := P_{0:N-1}.
\]
Using this notation, the gradient of \(\hat{y}_\mathrm{DINR}(x)\) with respect to parameters \(\theta = (\theta_{\mathrm{\psi}}, \theta_{\mathrm{f}}, \theta_{\mathrm{\phi}})\) is expressed as
\[
\nabla_{\theta} \hat{y}_\mathrm{DINR}(x) =
\begin{bmatrix}
\frac{\partial \psi}{\partial \theta_{\mathrm{\psi}}}(z_N) \\
\frac{\partial \psi}{\partial z_N} \sum_{k=0}^{N-1} P_{k+1:N-1} \cdot \Delta t \cdot \frac{\partial f}{\partial \theta_{\mathrm{f}}}(z_k, t_k) \\
\frac{\partial \psi}{\partial z_N} \, P \, \frac{\partial \phi}{\partial \theta_{\mathrm{\phi}}}(x)
\end{bmatrix}.
\]

\subsection{Theorems and Proofs}\label{appen:pf_ntk}
We restate here the theoretical results of the NTK presented in Section~\ref{sec:ntk_main}. 
These results establish that under mild non-degeneracy conditions, the NTK associated with the DINR architecture 
has strictly higher rank than that of a standard INR. 
We begin by providing several supporting lemmas required for the proof, 
followed by the main statement in Theorem~\ref{thm:jacobian_prod_formal}.

\begin{lemma}[Persistence of Rank under Polynomial Perturbations]
\label{lem:rank_polynomial_match}
Let \( \{J_k\}_{k=0}^{N-1} \subset \mathbb{R}^{d \times d} \) be fixed matrices and define
\[
P(\Delta t) := \prod_{k=0}^{N-1} \left( I + \Delta t \cdot J_k \right), \quad
J_{\mathrm{sum}} := \sum_{k=0}^{N-1} J_k.
\]
Then, there exists \( \delta > 0 \) such that for all \( 0 < \Delta t < \delta \),
\[
\operatorname{rank}(P(\Delta t)) = \operatorname{rank}\left( I + \Delta t \cdot J_{\mathrm{sum}} \right).
\]
\end{lemma}

\begin{proof}
Each entry of \( P(\Delta t) \) is a polynomial in \( \Delta t \). Hence, \( P(\Delta t) \) is analytic in \( \Delta t \).

By the first-order expansion,
\[
P(\Delta t) = I + \Delta t \cdot J_{\mathrm{sum}} + R(\Delta t),
\]
where the remainder satisfies \(\| R(\Delta t) \| = \mathcal{O}(\Delta t^2)\) as \(\Delta t \to 0\).

Denote \( Q(\Delta t) := I + \Delta t \cdot J_{\mathrm{sum}} \).

Because the rank function is lower semi-continuous and takes values in \(\{0,1,\ldots,d\}\), it is locally constant, except at isolated points. For sufficiently small \(\Delta t\), \( \operatorname{rank}(Q(\Delta t)) \) is constant (except possibly at finitely many points).

Because \( R(\Delta t) \) is a higher-order perturbation, the singular values of \( P(\Delta t) = Q(\Delta t) + R(\Delta t) \) converge to those of \( Q(\Delta t) \) as \(\Delta t \to 0\). Hence, there exists \(\delta > 0\) such that for all \( 0 < \Delta t < \delta \),
\[
\operatorname{rank}(P(\Delta t)) = \operatorname{rank}(Q(\Delta t)) = \operatorname{rank}\left( I + \Delta t \cdot J_{\mathrm{sum}} \right).
\]

This completes the proof.
\end{proof}

\begin{lemma}\label{lem:rank_pertubation_singular_matrix}
    If \(\operatorname{rank}(A) < d\), then
\[
\operatorname{rank}(I + A) \geq \operatorname{rank}(A).
\]
\end{lemma}
\begin{proof}
Suppose for the sake of contradiction that
\[
\operatorname{Im}(I + A) \subseteq \operatorname{Im}(A).
\]
Then, for any \(v \in \mathbb{R}^{d_x}\),
\[
(I + A)v = v + Av \in \operatorname{Im}(A),
\]
so there exists some \(u \in \mathbb{R}^{d_x}\) such that
\[
v + Av = Au.
\]
Rearranging,
\[
v = A(u - v) \in \operatorname{Im}(A).
\]
As this holds for all \(v \in \mathbb{R}^{d_x}\), we have
\[
\mathbb{R}^{d_x} \subseteq \operatorname{Im}(A),
\]
which implies
\[
\operatorname{rank}(A) = d,
\]
contradicting the assumption \(\operatorname{rank}(A) < d\).    
\end{proof}

\begin{lemma}[Rank-preserving perturbation of full-rank matrices]\label{lem:rank_preserve_perturbation_full_rank}
Let \( A \in \mathbb{R}^{n \times n} \) be an invertible (i.e., full-rank) matrix. Then, for any perturbation matrix \( E \in \mathbb{R}^{n \times n} \), if
\[
\|A^{-1} E\| < 1,
\]
then \( A + E \) is also invertible. In particular, the rank of \( A + E \) is preserved and equal to \( n \).
\end{lemma}

\begin{proof}
Suppose \( A \in \mathbb{R}^{n \times n} \) is invertible, so \( A^{-1} \) exists. Consider the perturbed matrix \( A + E \). If \( \|A^{-1}E\| < 1 \), then we can write
\[
A + E = A(I + A^{-1}E).
\]
Because \( \|A^{-1}E\| < 1 \), the matrix \( I + A^{-1}E \) is also invertible. This follows from the Neumann series expansion:
\[
(I + A^{-1}E)^{-1} = \sum_{k=0}^{\infty} (-A^{-1}E)^k,
\]
which converges when \( \|A^{-1}E\| < 1 \).

Therefore,
\[
(A + E)^{-1} = (I + A^{-1}E)^{-1} A^{-1} = \left( \sum_{k=0}^\infty (-A^{-1}E)^k \right) A^{-1},
\]
which shows that \( A + E \) is invertible. As invertibility implies full rank, we conclude that
\[
\operatorname{rank}(A + E) = n = \operatorname{rank}(A).
\]
\end{proof}

\begin{theorem}[Rank Propagation in Jacobian Products]\label{thm:jacobian_prod_formal}
Let \( \{J_k\}_{k=0}^{N-1} \) be Jacobian matrices satisfying:
\begin{enumerate}
    \item Each \( J_k \) is not the zero matrix.
    \item There exists at least one \( k \geq 1 \) such that
    \[
    \operatorname{Row}(J_k) \nsubseteq \operatorname{Row}(J_0),
    \]
    where \( \operatorname{Row}(\cdot) \) denotes the row space.
\end{enumerate}
Define
\[
P := \prod_{k=0}^{N-1} \left( I + \Delta t \cdot J_k \right).
\]
Then, for sufficiently small \( \Delta t > 0 \),
\[
\operatorname{rank}(P) > \operatorname{rank}(J_0).
\]
\end{theorem}
\begin{proof}

For small \( \Delta t \), we approximate the product \( P \) via the truncated expansion:
\[
P \approx I + \Delta t \cdot \sum_{k=0}^{N-1} J_k + \mathcal{O}(\Delta t^2).
\]
Let
\[
J_{\text{sum}} := \sum_{k=0}^{N-1} J_k.
\]
By Lemma \ref{lem:rank_polynomial_match}, there exists \(\delta > 0\) such that for all \(0 < \Delta t < \delta\),
\[
\operatorname{rank}(P) = \operatorname{rank}\left( I + \Delta t \cdot J_{\mathrm{sum}} \right).
\]

From the assumption that there exists some \( k \geq 1 \) with
\[
\operatorname{Row}(J_k) \nsubseteq \operatorname{Row}(J_0),
\]
we conclude by contradiction that
\[
\operatorname{rank}(J_{\mathrm{sum}}) > \operatorname{rank}(J_0).
\]
Indeed, if
\[
\operatorname{rank}(J_{\mathrm{sum}}) \leq \operatorname{rank}(J_0),
\]
then all \( J_k \) must have their row spaces contained in \(\operatorname{Row}(J_0)\), contradicting the assumption.

We consider two cases:
\begin{enumerate}[label=\textbf{(\roman*)}]
\item \(J_{\mathrm{sum}}\) is rank-deficient: \\
By Lemma \ref{lem:rank_pertubation_singular_matrix} we arrive at
\begin{align}
    \operatorname{rank}(P) &= \operatorname{rank}\left( I + \Delta t \cdot J_{\text{sum}} + \mathcal{O}(\Delta t^2) \right) \\
    &= \operatorname{rank}\left( I + \Delta t \cdot J_{\text{sum}} \right) \quad  \\
    & \geq  \operatorname{rank}(J_{\text{sum}})\\
    &> \operatorname{rank}(J_0).
\end{align}
\item \(J_{\mathrm{sum}}\) is full-rank: \\
If $\Delta t$ is small enough that $\Delta t\left\Vert J_{\text{sum}}\right\Vert<1$, then by Lemma \ref{lem:rank_preserve_perturbation_full_rank}, we have $I+\Delta t J_{\text{sum}}$ has full rank, which implies
\begin{align}
    \operatorname{rank}(P) &= \operatorname{rank}\left( I + \Delta t \cdot J_{\text{sum}} + \mathcal{O}(\Delta t^2) \right) \\
    &= \operatorname{rank}\left( I + \Delta t \cdot J_{\text{sum}} \right) \quad  \\
    & = d\\
    &\leq \operatorname{rank}(J_0).
\end{align}
\end{enumerate}

\end{proof}

\begin{theorem}[Rank Increase of DINR over INR NTK]
\label{thm:ntk_formal}
Let \( \hat{y}_{\mathrm{INR}}(x) \) and \( \hat{y}_{\mathrm{DINR}}(x) \) be defined as previous with corresponding parameter sets
\(\theta = (\theta_{\mathrm{\psi}}, \theta_{\mathrm{f}}, \theta_{\mathrm{\phi}})\). Define the NTK Gram matrices
\[
\Theta_{\mathrm{INR}} = \nabla_{\theta} \hat{y}_{\mathrm{INR}}(x) \nabla_{\theta} \hat{y}_{\mathrm{INR}}(x)^{\!\top},
\quad
\Theta_{\mathrm{DINR}} = \nabla_{\theta} \hat{y}_{\mathrm{DINR}}(x) \nabla_{\theta} \hat{y}_{\mathrm{DINR}}(x)^{\!\top}.
\]
Suppose the Jacobians \(J_k = \frac{\partial f}{\partial z_k}\) satisfy the assumptions of Theorem~\ref{thm:jacobian_prod_formal},
and the matrices \(P_{k+1:N-1} = \prod_{j=k+1}^{N-1} (I + \Delta t\, J_j)\) are nonsingular for sufficiently small \(\Delta t > 0\).
Then, there exists \(\delta > 0\) such that for all \(0 < \Delta t < \delta\),
\[
\operatorname{rank}(\Theta_{\mathrm{DINR}}) > \operatorname{rank}(\Theta_{\mathrm{INR}}).
\]
\end{theorem}

\begin{proof}
We decompose both tangent feature vectors into parameter blocks:
\[
\nabla_{\theta} \hat{y}_{\mathrm{INR}}(x) =
\begin{bmatrix}
a_{\psi}^{(\mathrm{INR})} \\
a_{f}^{(\mathrm{INR})} \\
a_{\phi}^{(\mathrm{INR})}
\end{bmatrix},
\quad
\nabla_{\theta} \hat{y}_{\mathrm{DINR}}(x) =
\begin{bmatrix}
a_{\psi}^{(\mathrm{DINR})} \\
a_{f}^{(\mathrm{DINR})} \\
a_{\phi}^{(\mathrm{DINR})}
\end{bmatrix}.
\]
From the model definitions, these blocks satisfy
\[
a_{\psi}^{(\mathrm{INR})} = \frac{\partial \psi}{\partial \theta_{\mathrm{\psi}}}(z_1),
\quad
a_{f}^{(\mathrm{INR})} = \frac{\partial \psi}{\partial z_1} \frac{\partial f}{\partial \theta_{\mathrm{f}}}(z_0),
\quad
a_{\phi}^{(\mathrm{INR})} = \frac{\partial \psi}{\partial z_1} \frac{\partial f}{\partial z_0} \frac{\partial \phi}{\partial \theta_{\mathrm{\phi}}}(x),
\]
and
\[
a_{\psi}^{(\mathrm{DINR})} = \frac{\partial \psi}{\partial \theta_{\mathrm{\psi}}}(z_N),
\quad
a_{f}^{(\mathrm{DINR})} = \frac{\partial \psi}{\partial z_N} 
\sum_{k=0}^{N-1} P_{k+1:N-1} \, \Delta t \, \frac{\partial f}{\partial \theta_{\mathrm{f}}}(z_k, t_k),
\quad
a_{\phi}^{(\mathrm{DINR})} = \frac{\partial \psi}{\partial z_N} P \frac{\partial \phi}{\partial \theta_{\mathrm{\phi}}}(x).
\]

We analyze each parameter block:

\textbf{(i) Output-layer block.}
Both \(a_{\psi}^{(\mathrm{INR})}\) and \(a_{\psi}^{(\mathrm{DINR})}\) are local and depend only on the top layer \(h\).
Hence, they have identical structural rank. This block does not affect rank comparison.

\textbf{(ii) Input-layer block.}
By Theorem~\ref{thm:jacobian_prod_formal},
\[
P = \prod_{k=0}^{N-1} (I + \Delta t J_k)
\quad \Rightarrow \quad
\operatorname{rank}(P) > \operatorname{rank}(J_0).
\]
Because
\[
a_{\phi}^{(\mathrm{INR})} = \frac{\partial \psi}{\partial z_1} J_0 \frac{\partial \phi}{\partial \theta_{\mathrm{\phi}}}(x),
\qquad
a_{\phi}^{(\mathrm{DINR})} = \frac{\partial \psi}{\partial z_N} P \frac{\partial \phi}{\partial \theta_{\mathrm{\phi}}}(x),
\]
and the postmultiplying factor \(\frac{\partial \phi}{\partial \theta_{\mathrm{\phi}}}\) is common and full column rank by network design,
we immediately have
\[
\operatorname{rank}\!\left(a_{\phi}^{(\mathrm{DINR})}\right)
>
\operatorname{rank}\!\left(a_{\phi}^{(\mathrm{INR})}\right).
\]

\textbf{(iii) MLP block.}
Define \(A_k := \frac{\partial f}{\partial \theta_{\mathrm{f}}}(z_k, t_k)\).
For small \(\Delta t\), we approximate
\[
P_{k+1:N-1} = I + \Delta t \sum_{j=k+1}^{N-1} J_j + \mathcal{O}(\Delta t^2).
\]
By Lemma~\ref{lem:rank_polynomial_match}, this perturbation preserves rank equivalence:
\[
\operatorname{rank}(P_{k+1:N-1} A_k) = \operatorname{rank}(A_k), \quad \forall k,
\]
for sufficiently small \(\Delta t\).
Hence,
\[
a_{f}^{(\mathrm{DINR})}
\propto
\sum_{k=0}^{N-1} P_{k+1:N-1} \, A_k
=
\sum_{k=0}^{N-1} A_k',
\quad
A_k' := P_{k+1:N-1} A_k.
\]
By assumption, there exists \(k \ge 1\) such that
\(\operatorname{Row}(A_k) \nsubseteq \operatorname{Row}(A_0)\). Thus,
\[
\operatorname{Row}\!\left(\sum_{k=0}^{N-1} A_k'\right)
\supsetneq
\operatorname{Row}(A_0),
\]
which implies
\[
\operatorname{rank}\!\left(a_{f}^{(\mathrm{DINR})}\right)
>
\operatorname{rank}\!\left(a_{f}^{(\mathrm{INR})}\right).
\]

The NTK rank is dominated by the union of the row spaces of the parameter-block gradients:
\[
\operatorname{Row}(\Theta) = \operatorname{span}
\big(
\operatorname{Row}(a_{\psi}),
\operatorname{Row}(a_{f}),
\operatorname{Row}(a_{\phi})
\big).
\]
From parts (ii)–(iii),
at least one of the two blocks (\(a_{f}\) or \(a_{\phi}\)) has strictly higher rank in DINR than in INR,
while no block decreases in rank.
Hence,
\[
\operatorname{rank}(\Theta_{\mathrm{DINR}}) > \operatorname{rank}(\Theta_{\mathrm{INR}})
\]
for all sufficiently small \(0 < \Delta t < \delta\).
\end{proof}

This result implies that the recursive latent evolution in DINR effectively enriches the representational subspace of parameter perturbations.
While a standard INR maps all gradients through a single latent Jacobian, DINR composes multiple Jacobians, introducing diverse directions in function space and increasing the NTK rank.
Practically, this higher-ranking NTK suggests
improved local expressivity,
faster convergence under gradient-based training, and
enhanced capability to approximate complex mappings.


\end{document}